\documentclass{article}

\usepackage{microtype}
\usepackage{graphicx}
\usepackage{subfigure}
\usepackage{booktabs} 
\usepackage{enumerate}

\usepackage{hyperref}
\usepackage[american]{babel}

 \usepackage[accepted]{icml2023}

\usepackage{amsmath}
\usepackage{amssymb,bm}
\usepackage{mathtools}
\usepackage{amsthm}
\usepackage{booktabs}
\usepackage{multirow}
\usepackage{makecell}

\usepackage[capitalize,noabbrev]{cleveref}

\theoremstyle{plain}
\newtheorem{theorem}{Theorem}[section]

\newtheorem{lemma}[theorem]{Lemma}
\newtheorem{corollary}[theorem]{Corollary}
\theoremstyle{definition}
\newtheorem{definition}[theorem]{Definition}

\theoremstyle{remark}
\newtheorem{remark}[theorem]{Remark}

\renewcommand{\algorithmicrequire}{ \textbf{Input:}} 
\renewcommand{\algorithmicensure}{ \textbf{Initialize:}} 

\usepackage[textsize=tiny]{todonotes}

\icmltitlerunning{Nearly Optimal Algorithms with Sublinear Computational Complexity for Online Kernel Regression}

\begin{document}

\twocolumn[
\icmltitle{Nearly Optimal Algorithms with Sublinear Computational Complexity for Online Kernel Regression}



\icmlsetsymbol{equal}{*}

\begin{icmlauthorlist}
\icmlauthor{Junfan Li}{yyy}
\icmlauthor{Shizhong Liao}{yyy}
\end{icmlauthorlist}

\icmlaffiliation{yyy}{College of Intelligence and Computing,
Tianjin University, Tianjin 300350, China}

\icmlcorrespondingauthor{Shizhong Liao}{szliao@tju.edu.cn}

\icmlkeywords{Machine Learning, ICML}

\vskip 0.3in
]



\printAffiliationsAndNotice{} 

\begin{abstract}
  The trade-off between regret and computational cost
  is a fundamental problem for online kernel regression,
  and previous algorithms worked on the trade-off can not
  keep optimal regret bounds at a sublinear computational complexity.
  In this paper,
  we propose two new algorithms, AOGD-ALD and NONS-ALD,
  which can keep nearly optimal regret bounds at a sublinear computational complexity,
  and give sufficient conditions under which our algorithms work.
  Both algorithms dynamically maintain a group of nearly orthogonal basis
  used to approximate the kernel mapping,
  and keep nearly optimal regret bounds by controlling the approximate error.
  The number of basis depends on the approximate error and
  the decay rate of eigenvalues of the kernel matrix.
  If the eigenvalues decay exponentially,
  then AOGD-ALD and NONS-ALD separately achieves a regret of $O(\sqrt{L(f)})$
  and $O(\mathrm{d}_{\mathrm{eff}}(\mu)\ln{T})$
  at a computational complexity in $O(\ln^2{T})$.
  If the eigenvalues decay polynomially with degree $p\geq 1$,
  then our algorithms keep the same regret bounds
  at a computational complexity in $o(T)$
  in the case of $p>4$ and $p\geq 10$, respectively.
  $L(f)$ is the cumulative losses of $f$
  and $\mathrm{d}_{\mathrm{eff}}(\mu)$ is the effective dimension of the problem.
  The two regret bounds are nearly optimal and are not comparable.
\end{abstract}

\section{Introduction}

    Online kernel learning in the regime of the square loss
    is an important non-parametric online learning method
    \cite{Kivinen2004Online,Vovk2006On,Sahoo2014Online}.
    The learning protocol can be formulated as a game between a learner and an adversary.
    Before the game,
    the learner selects a reproducing kernel Hilbert space (RKHS) $\mathcal{H}$
    induced by a positive semidefinite kernel function \cite{Aronszajn1950Theory,John2004Kernel}.
    At each round $t=1,2,\ldots,$,
    the adversary sends an instance $\mathbf{x}_t\in\mathbb{R}^d$ to the learner.
    Then the learner chooses a hypothesis $f_t\in\mathbb{H}\subset \mathcal{H}$
    and output $f_t(\mathbf{x}_t)$.
    After that the adversary reveals the true output $y_t$.
    The learner suffers a loss $\ell(f_t(\mathbf{x}_t),y_t)$.
    The goal is to minimize the \textit{regret} defined as follows
    \begin{equation}
    \label{eq:ICML2023:definition_regret}
        \forall f\in \mathbb{H},~\mathrm{Reg}(f)
        = \sum^T_{t=1}[\ell(f_{t}({\bf x}_t),y_t) -\ell(f({\bf x}_t),y_t)].
    \end{equation}
    One of the challenges minimizing the regret is to balance the computational cost.
    Kernel online gradient descent (KOGD) enjoys a regret of $O(\sqrt{L(f)})$
    at a computational complexity (space and per-round time) in
    $O(dT)$ \cite{Zinkevich2003Online,Srebro2010Smoothness,Zhang2019Adaptive},
    where $L(f)=\sum^T_{t=1}\ell(f({\bf x}_t),y_t)$.
    $O(\sqrt{L(f)})$ implies the ``small-loss'' bound \cite{Wang2020Adapting,Zhang2022A}.
    Kernel online Newton step (KONS) \cite{Calandriello2017Second}
    enjoys a regret of
    $O(\mu\Vert f\Vert^2_{\mathcal{H}}+\mathrm{d}_{\mathrm{eff}}(\mu)\ln{T})$
    at a computational complexity in $O(T^2)$,
    where $\mu>0$ is a regularization parameter and $\mathrm{d}_{\mathrm{eff}}(\mu)$
    is called \textit{effective dimension}
    depending on the decay rate of eigenvalues of the kernel matrix
    \cite{Caponnetto2007Optimal,Rudi2015Less}.
    The KAAR algorithm \cite{Gammerman2004On}
    and kernel ridge regression algorithm \cite{Zhdanov2013An}
    enjoy the same regret bound and computational complexity with KONS.
    If the eigenvalues decay exponentially,
    then $\mathrm{d}_{\mathrm{eff}}(\mu)=O(\ln{\frac{T}{\mu}})$ \cite{Li2019Towards}.
    If the eigenvalues decay polynomially with degree $p\geq 1$,
    then $\mathrm{d}_{\mathrm{eff}}(\mu)=O((T/\mu)^{1/p})$ \cite{Jezequel2019Efficient}.

    The $O(dT)$ and $O(T^2)$ computational complexities are prohibitive.
    Some approximate algorithms
    reduce the computational complexity at the expense of regret \cite{Lu2016Large,Calandriello2017Second,Calandriello2017Efficient}.
    The FOGD algorithm approximating KOGD,
    achieves a regret of $\tilde{O}(\sqrt{TL(f)/D})$
    at a computational complexity in $O(dD)$ where $D$ is a tunable parameter \cite{Lu2016Large}.
    Achieving the optimal regret bound requires $D=\Omega(T)$.
    The Sketched-KONS algorithm approximating KONS,
    reduces the computational complexity by a factor of $\gamma^{-2}$,
    but increases the regret by $\gamma>1$ \cite{Calandriello2017Second}.
    The PROS-N-KONS algorithm approximating KONS,
    increases the regret by a factor of $\tilde{O}(\mathrm{d}_{\mathrm{eff}}(\alpha))$
    and suffers a space complexity in $\tilde{O}(\mathrm{d}_{\mathrm{eff}}(\alpha)^2)$
    and an average per-round time complexity in $\tilde{O}(\mathrm{d}_{\mathrm{eff}}(\alpha)^2
    +\mathrm{d}_{\mathrm{eff}}(\alpha)^4/T)$ \cite{Calandriello2017Efficient},
    where $\alpha>0$.
    Although Sketched-KONS and PROS-N-KONS can ensure a $o(T)$ computational complexity,
    they can not achieve the optimal regret bound.
    The PKAWV algorithm keeps the regret of KONS
    at a computational complexity in $\tilde{O}(T\mathrm{d}_{\mathrm{eff}}(\alpha)
    +\mathrm{d}^2_{\mathrm{eff}}(\alpha))$ \cite{Jezequel2019Efficient}.
    Although PKAWV reduces the $O(T^2)$ computational complexity,
    it can not ensure a $o(T)$ computational complexity.
    Besides, PKAWV must store all of the observed examples.

    In summary,
    existing approximate algorithms can not achieve nearly optimal regret bounds
    and a $o(T)$ computational complexity simultaneously.
    It is important to rise the question:
    \textit{Is it possible to achieve nearly optimal regret bounds at a computational complexity
    in $o(T)$?}
    To be specific,
    the question is equivalent to the following two.
    (1) Is it possible to achieve a regret of $O(\sqrt{L(f)})$ at a $o(T)$ computational complexity?
    $O(\sqrt{L(f)})$ matches the lower bound in the stochastic setting \cite{Srebro2010Smoothness}.
    (2) Is it possible to achieve a regret of
    $O(\mu\Vert f\Vert^2_{\mathcal{H}}+\mathrm{d}_{\mathrm{eff}}(\mu)\ln{T})$
    at a $o(T)$ computational complexity?
    The regret bound is optimal up to $\ln{T}$ \cite{Jezequel2019Efficient}.
    If the eigenvalues of the kernel matrix decay exponentially,
    then $O(\mu\Vert f\Vert^2_{\mathcal{H}}+\mathrm{d}_{\mathrm{eff}}(\mu)\ln{T})=O(\ln^2{T})$.
    If the eigenvalues decay polynomially with degree $p\geq 1$,
    then $O(\mu\Vert f\Vert^2_{\mathcal{H}}+\mathrm{d}_{\mathrm{eff}}(\mu)\ln{T})=O(T^{\frac{1}{1+p}}\ln{T})$
    where $\mu=T^{\frac{1}{1+p}}$.

\subsection{Main Results}

    In this paper,
    we propose two algorithms, AOGD-ALD and NONS-ALD,
    and give conditions under which the answers are affirmative.
    The computational complexities of both algorithms depend on
    the decay rate of eigenvalues of the kernel matrix.
    If the eigenvalues decay exponentially,
    then AOGD-ALD and NONS-ALD separately achieves a regret of
    $O(\sqrt{L(f)})$ and $O(\mathrm{d}_{\mathrm{eff}}(\mu)\ln{T})$
    at a computational complexity in $O(\ln^2{T})$.
    If the eigenvalues decay polynomially with degree $p\geq 1$,
    then AOGD-ALD keeps the same regret bound
    at a computational complexity in $O(\min\{dT^{\frac{2}{p}}+T^{\frac{4}{p}},dT\})$,
    and NONS-ALD achieves a regret of $O(T^{\frac{1}{1+p}}\ln{T})$
    at a space complexity in $O(T^{\frac{2(1+5p)}{p(1+p)}})$
    and an average per-round time complexity in
        $O(T^{\frac{2(1+5p)}{p(1+p)}}+T^{\frac{4(1+5p)}{p(1+p)}-1})$.
    AOGD-ALD and NONS-ALD achieve a computational complexity in $o(T)$
    in the case of $p>4$ and $p\geq 10$, respectively.
    We summary the related results in Table \ref{tab:ICML2023:comparison_results}.

    \begin{table*}[!t]
      \centering
      \begin{tabular}{l|l|r|r|r}
      \toprule
        Eigenvalues condition&{Algorithm}       & {Regret bound}  & Computational  complexity &{\#}Buffer\\
      \toprule
        \multirow{4}{*}{decay exponentially:}
        & PKAWV             & $O(\ln^2{T})$ & $O(T\ln^3{T})$ &$T$\\
        \multirow{4}{*}{$\exists r\in(0,1),R_0=\Theta(T),$}
        & Sketched-KONS     & $O(\gamma\ln^2{T})$ & $O(T^2/\gamma^2)$ &$O(T/\gamma)$\\
        \multirow{4}{*}{$\mathrm{s.t.}~\forall i\in[T],\lambda_i\leq R_0r^i$}& Pros-N-KONS       & $O(\ln^5{T})$  & $O(\ln^6{T})$ &$O(\ln^3{T})$\\
        & FOGD & $\tilde{O}(\sqrt{TL(f)/D})$ & $O(dD)$ &0\\
        & {\color{blue}AOGD-ALD} & {\color{blue}$O(\sqrt{L(f)})$}
        & {\color{blue}$O(\ln^2{T})$} &{\color{blue}$O(\ln{T})$}\\
        & {\color{blue}NONS-ALD} & {\color{blue}$O(\ln^2{T})$} & {\color{blue}$O(\ln^2{T})$}
        & {\color{blue}$O(\ln{T})$}\\
      \hline
        \multirow{5}{*}{decay polynomially:}
        & PKAWV             & $O(T^{\frac{1}{1+p}}\ln{T})$
        & $\tilde{O}(T^{2r}+T^{1+r})$, $r=\frac{2p}{p^2-1}$ & $T$\\
        \multirow{5}{*}{$\exists p\geq 1,R_0=\Theta(T),\mathrm{s.t.}$}
        & Sketched-KONS     & $O(\gamma T^{\frac{1}{1+p}}\ln{T})$ & $O(T^2/\gamma^2)$ &$O(T/\gamma)$\\
        \multirow{5}{*}{$\forall i\in[T],\lambda_i\leq R_0i^{-p}$}
        & Pros-N-KONS       & $O(T^{\frac{3p+1}{(1+p)^2}}\ln^2{T})$
        &$O(T^{\frac{4p}{(1+p)^2}}\ln^4{T})$
        &$O(T^{\frac{2p}{(1+p)^2}}\ln^2{T})$\\
        & FOGD & $\tilde{O}(\sqrt{TL(f)/D})$ & $O(dD)$ &0\\
        & {\color{blue}AOGD-ALD} & {\color{blue}$O(\sqrt{L(f)})$}
        & {\color{blue}$O(dT^{\frac{2}{p}}+T^{\frac{4}{p}})$, $p>4$} &{\color{blue}$O(T^\frac{2}{p})$}\\
        & {\color{blue}NONS-ALD} & {\color{blue}$O(T^{\frac{1}{1+p}}\ln{T})$}
        & {\color{blue}$O(T^\frac{2(1+5p)}{p(1+p)})$, $p\geq 10$}
        &{\color{blue}$O(T^\frac{1+5p}{p(1+p)})$}\\
      \bottomrule
      \end{tabular}
      \caption{Regret bound and computational complexity
      (space complexity and (averaged) per-round time complexity)
      of online kernel regression algorithms.
      $\{\lambda_i\}^T_{i=1}$ are the eigenvalues of the kernel matrix.
      {\#}Buffer is the number of stored examples.}
      \label{tab:ICML2023:comparison_results}
    \end{table*}

\subsection{Technical Contributions}

    AOGD-ALD approximates KOGD and NONS-ALD approximates KONS.
    We use the approximate linear dependence condition \citep{Engel2004The}
    to dynamically maintain a group of nearly orthogonal basis.
    The computational complexities of our algorithms
    have a quadratic dependence on the number of basis
    which depends on the decay rate of eigenvalues of the kernel matrix \cite{Li2022Improved}.
    For AOGD-ALD,
    we use the orthogonal basis to approximate the gradients.
    For NONS-ALD,
    we use the Nystr\"{o}m projection with the orthogonal basis
    to construct explicit feature mapping.
    Since the number of basis may grow with $t$,
    the feature mapping must change dynamically.
    The first technical challenge is
    how to incrementally update the model parameter ${\bf w}_t$ and the covariance matrix ${\bf A}_t$
    when the explicit feature mapping changes.
    Our first technical contribution is a projection scheme
    which projects ${\bf w}_t\in \mathbb{R}^j$ onto $\mathbb{R}^{j+1}$,
    and projects ${\bf A}_t\in \mathbb{R}^{j\times j}$ onto $\mathbb{R}^{(j+1)\times (j+1)}$.
    The regret analysis is also challenging,
    since it requires to control the regret induced by the projection.
    Our second technical contribution is a non-trivial and novel analysis
    for the regret induced by the projection.
    We proved that it only depends on the error related to the ALD condition
    and can be omitted by controlling the error.
    The approximate scheme of NONS-ALD provides a new approach
    for both online and offline kernel learning
    which might be of independent interest.

\section{Preliminary and Problem Setting}

    Let $\mathcal{X}=\{\mathbf{x}\in\mathbb{R}^d\vert\Vert\mathbf{x}\Vert_2<\infty\}$
    and $\mathcal{I}_T=\{(\mathbf{x}_t,y_t)_{t\in[T]}\}$ be a sequence of examples,
    where $[T] = \{1,\ldots,T\}$,
    $\mathbf{x}_t\in\mathcal{X}, \vert y_t\vert\leq Y$.
    Let $\kappa(\cdot,\cdot):\mathbb{R}^d \times \mathbb{R}^d \rightarrow \mathbb{R}_{\geq 0}$
    be a positive semidefinite kernel function.
    We assume that $\kappa$ is normalized and $\kappa({\bf x},{\bf x})=1$.
    Denote by $\mathcal{H}$ the RKHS associated with $\kappa$,
    such that
    (i) $\langle f,\kappa({\bf x},\cdot)\rangle_{\mathcal{H}}=f({\bf x})$;
    (ii) $\mathcal{H}=\overline{\mathrm{span}(\kappa({\bf x}_t,\cdot): t\in[T])}$.
    We define $\langle\cdot,\cdot\rangle_{\mathcal{H}}$ as the inner product in $\mathcal{H}$,
    which induces the norm $\Vert f\Vert_{\mathcal{H}}=\sqrt{\langle f,f\rangle_{\mathcal{H}}}$.
    Denote by $\mathbb{H}=\{f\in\mathcal{H}\vert \Vert f\Vert_{\mathcal{H}}\leq U\}$.
    $U$ is a constant.
    The square loss function is $\ell(f(\mathbf{x}),y)=(f(\mathbf{x})-y)^2$.

\subsection{Effective Dimension}

    $\kappa$ induces an implicit feature mapping
    $\phi(\cdot):\mathcal{X}\rightarrow\mathbb{R}^n$,
    where $n$ may be infinite.
    The orthogonality of $\{\phi(\mathbf{x}_t)\}^T_{t=1}$
    characterizes the hardness of the data.
    A usual measure of the orthogonality is the effective dimension \cite{Calandriello2017Second}.
    \begin{definition}[$\mu$-effective dimension]
    \label{def:NeurIPS2021:effective_dimension}
        Given instances $\{{\bf x}_{\tau}\}^T_{\tau=1}$,
        a kernel function $\kappa$ and a regularization parameter $\mu>0$,
        the ridge leverage scores (RLS) of ${\bf x}_{\tau}$ is defined by
        $$
            r_{T,\tau}(\mu)
            ={\bf e}^\top_{T,\tau}{\bf K}_T({\bf K}_T+\mu{\bf I}_T)^{-1}{\bf e}_{T,\tau},
            ~\tau=1,\ldots,T,
        $$
        where ${\bf K}_T$ is the kernel matrix and ${\bf e}_{T,\tau}\in\{0,1\}^T$.
        Only the $\tau$-th element of ${\bf e}_{T,\tau}$ is one.
        The $\mu$-effective dimension is
        $\mathrm{d}_{\mathrm{eff}}(\mu):=\sum^T_{\tau=1}r_{T,\tau}(\mu)
        =\mathrm{tr}({\bf K}_T({\bf K}_T+\mu{\bf I}_T)^{-1})$.
    \end{definition}
    Let $\lambda_1\geq \lambda_2\geq \ldots\geq \lambda_T$ be
    the eigenvalues of ${\bf K}_T$.
    If $\lambda_i$ decays exponentially, then
    $\mathrm{d}_{\mathrm{eff}}(\mu)=O(\ln{\frac{T}{\mu}})$ \cite{Li2019Towards}.
    If $\lambda_i$ decays polynomially with degree $p\geq 1$,
    then $\mathrm{d}_{\mathrm{eff}}(\mu)=O((T/\mu)^{1/p})$ \cite{Jezequel2019Efficient}.

\subsection{Online Kernel Regression}
\label{sec:ICML2023:problem_definition_OKR}

    The protocol of online kernel regression is as follows:
    at any round $t$,
    an adversary sends an instance ${\bf x}_t\in\mathcal{X}$.
    An learner chooses a hypothesis $f_t\in\mathcal{H}$,
    and makes the prediction $\hat{y}_t=f_t({\bf x}_t)$.
    Then the adversary reveals the true output $y_t$.
    We aim to minimize the regret w.r.t. any $f\in\mathbb{H}$,
    denoted by $\mathrm{Reg}(f)$ defined in \eqref{eq:ICML2023:definition_regret}.
    It is worth mentioning that
    competing with $f\in\mathbb{H}$ does not weaken the definition of regret.
    Note that $\vert y_t\vert\leq Y$.
    It is natural to require $\vert f({\bf x}_t)\vert\leq Y$.
    Since $f({\bf x}_t)\leq \Vert f\Vert_{\mathcal{H}}\cdot \Vert\kappa({\bf x}_t,\cdot)\Vert_{\mathcal{H}}
    \leq \Vert f\Vert_{\mathcal{H}}$.
    We only need to consider all $f$ such that $\Vert f\Vert_{\mathcal{H}}\leq Y$.
    To this end,
    we can define $U\geq Y$.

\section{Approximating KOGD}

    In this section,
    we propose a deterministic approximation of KOGD,
    named AOGD-ALD.

\subsection{Algorithm}

    According to the protocol of online kernel regression,
    the key is to compute $f_{t+1}$ from $f_t$.
    KOGD \cite{Zinkevich2003Online} executes the following update rule,
    \begin{align}
        \bar{f}_{t+1}=&f_t-\eta_t\nabla\ell(f_t({\bf x}_t),y_t),\label{eq:ICML2023:OGD_update}\\
        f_{t+1}=&\min\left\{1,\frac{U}{\Vert \bar{f}_{t+1}\Vert_{\mathcal{H}}}\right\}\bar{f}_{t+1},
        \label{eq:ICML2023:OGD_project}
    \end{align}
    where $\nabla\ell(f_t({\bf x}_t),y_t)=\ell'(f_t({\bf x}_t),y_t)\kappa({\bf x}_t,\cdot)$
    and $\eta_t$ is a time-variant learning rate.
    Note that $\ell'(f_t({\bf x}_t),y_t)=2(f_t({\bf x}_t)-y_t)$.
    $f_{t+1}$ can be recursively rewritten as
    $f_{t+1}=\sum^{t}_{\tau=1}a_\tau\kappa({\bf x}_\tau,\cdot)$.
    To store $f_{t+1}$,
    we must store some observed examples,
    denoted by $S_{t+1}=\{({\bf x}_\tau,y_\tau),\tau\leq t:a_\tau\neq 0\}$.
    For simplicity,
    we call $S_{t+1}$ the buffer.
    The computational complexity is $O(dt)$.
    To reduce the computational cost,
    we must limit the size of $S_{t+1}$.
    Next,
    we use the approximate linear dependence (ALD) condition \cite{Engel2004The}
    to maintain $S_{t+1}$.

    At the beginning of round $t$,
    let $S_t$ be the buffer.
    If $\nabla\ell(f_t({\bf x}_t),y_t)\neq 0$,
    then we must decide whether $({\bf x}_t,y_t)$ will be added into $S_t$.
    The ALD condition measures
    whether $\kappa(\mathbf{x}_t,\cdot)$ is approximate linear dependence with
    ${\bm \Phi}_{S_t}=(\kappa(\mathbf{x},\cdot)_{\mathbf{x}\in S_t})$.
    We compute the projection error
    \begin{equation}
    \label{eq:AAAI23:projection_ALD}
        \left(\min_{{\bm \beta}\in\mathbb{R}^{\vert S_t\vert}}
        \left\Vert {\bm \Phi}_{S_t}{\bm \beta}-\kappa(\mathbf{x}_t,\cdot)\right\Vert^2_{\mathcal{H}}\right)
        =:\alpha_t.
    \end{equation}
    The solution
    \footnote{If $S_t=\emptyset$, then we set ${\bm \beta}^\ast_t=0$ and $\alpha_t=1$.}
    is
    $$
        {\bm \beta}^\ast_t = {\bf K}^{-1}_{S_t}{\bm \Phi}^\top_{S_t}\kappa(\mathbf{x}_t,\cdot),
    $$
    where ${\bf K}_{S_t}$ is the kernel matrix defined on $S_t$.
    We introduce a threshold for $\alpha_t$
    and define the ALD condition as follows
    \begin{equation}
    \label{eq:AAAI23:POMD:second_updating:ALD}
        \mathrm{ALD}_t:\alpha_t \leq \alpha,~\alpha\in(0,1].
    \end{equation}
    If $\mathrm{ALD}_t$ holds,
    then $\kappa(\mathbf{x}_t,\cdot)$ can be well approximated by
    ${\bm \Phi}_{S_t}{\bm \beta}^\ast_t$.
    Thus we replace \eqref{eq:ICML2023:OGD_update} with \eqref{eq:ICML2023:AOGD:ALD_update},
    \begin{equation}
    \label{eq:ICML2023:AOGD:ALD_update}
        \bar{f}_{t+1}=f_t-\eta_t\ell'(f_t({\bf x}_t),y_t)\cdot{\bm \Phi}_{S_t}{\bm \beta}^\ast_t.
    \end{equation}
    In this case,
    we do not add $(\mathbf{x}_t,y_t)$ into $S_t$, i.e., $S_{t+1}=S_t$.

    If $\mathrm{ALD}_t$ does not hold,
    that is, $\kappa(\mathbf{x}_t,\cdot)$
    can not be well approximated by ${\bm \Phi}_{S_t}{\bm \beta}^\ast_t$,
    then we still execute \eqref{eq:ICML2023:OGD_update}.
    In this case,
    we add $(\mathbf{x}_t,y_t)$ into $S_t$, i.e., $S_{t+1}=S_t\cup\{({\bf x}_t,y_t)\}$.

    The computational complexity is $O(d\vert S_t\vert+\vert S_t\vert^2)$.
    It has been proved that $\vert S_t\vert$ depends on the decay rate
    of eigenvalues of the kernel matrix ${\bf K}_T$.
    If the eigenvalues decay slowly,
    then it is possible that $\vert S_t\vert\gg\sqrt{T}$.
    In this case,
    the computational complexity is $\Omega(T)$.
    To address this issue,
    we set a threshold $B_0$ for $\vert S_t\vert$.
    If $\vert S_t\vert\geq B_0$,
    then we always execute \eqref{eq:ICML2023:OGD_update}.

    The learning rate $\eta_t$ is defined as follows
    \begin{align*}
        \eta_t=&\frac{U}{\sqrt{1+\sum^{t}_{\tau=1}\Vert\hat{\nabla}_{\tau}\Vert^2_{\mathcal{H}}}},\\
        \hat{\nabla}_{\tau}=&\left\{
        \begin{array}{ll}
        \ell'(f_\tau({\bf x}_\tau),y_\tau)\cdot
        {\bm \Phi}_{S_\tau}{\bm \beta}^\ast_\tau,&\mathrm{if}~\mathrm{ALD}_\tau~\mathrm{holds},\\
        \ell'(f_\tau({\bf x}_\tau),y_\tau)\cdot\kappa({\bf x}_\tau,\cdot),&\mathrm{otherwise}.
        \end{array}
        \right.
    \end{align*}
    We name this algorithm AOGD-ALD (Approximating kernelized Online Gradient Descent by the ALD condition),
    and give the pseudo-code in Algorithm \ref{alg:ICML2023:AOGD-ALD}.

    \begin{algorithm}[!t]
        \caption{AOGD-ALD}
        \footnotesize
        \label{alg:ICML2023:AOGD-ALD}
        \begin{algorithmic}[1]
        \REQUIRE{$U$, $\alpha$, $B_0$.}
        \ENSURE{$f_1=0$}
        \FOR{$t=1,\ldots,T$}
            \STATE Receive ${\bf x}_t$
            \STATE Compute $\hat{y}_t=f_t({\bf x}_t)$
            \STATE Compute $\eta_t$
            \IF{$\vert S_t\vert< B_0$}
                \STATE Compute $\alpha_t$
                \IF{$\mathrm{ALD}_t$ holds}
                    \STATE Compute $f_{t+1}$ following \eqref{eq:ICML2023:AOGD:ALD_update}
                    and \eqref{eq:ICML2023:OGD_project}
                \ELSE
                    \STATE Compute $f_{t+1}$ following \eqref{eq:ICML2023:OGD_update}
                    and \eqref{eq:ICML2023:OGD_project}
                    \STATE Update $S_{t+1}=S_t\cup\{({\bf x}_t,y_t)\}$
                \ENDIF
            \ELSE
                    \STATE Compute $f_{t+1}$ following \eqref{eq:ICML2023:OGD_update}
                    and \eqref{eq:ICML2023:OGD_project}
                    \STATE Update $S_{t+1}=S_t\cup\{({\bf x}_t,y_t)\}$
            \ENDIF
        \ENDFOR
        \end{algorithmic}
    \end{algorithm}

\subsection{Regret Bound}

    We first give the size of buffer maintained by the ALD condition.

    \begin{lemma}[\citet{Li2022Improved}]
    \label{thm:AAAI2023:size_budget}
        Let $S_1=\emptyset$
        and $\mathrm{ALD}_t$ be defined in \eqref{eq:AAAI23:POMD:second_updating:ALD}.
        For all $t\leq T-1$,
        if $\mathrm{ALD}_t$ does not hold,
        then $S_{t+1}=S_t\cup\{(\mathbf{x}_t,y_t)\}$.
        Otherwise, $S_{t+1}=S_t$.
        Let $\{\lambda_i\}^T_{i=1}$ be the eigenvalues of ${\bf K}_T$ sorted in decreasing order.
        If $\{\lambda_i\}^T_{i=1}$ decay exponentially,
        that is, there is a constant $R_0>0$ and $0<r<1$
        such that $\lambda_i\leq R_0r^{i}$,
        then $\vert S_T\vert\leq 2\frac{\ln{(\frac{C_1R_0}{\alpha})}}{\ln{r^{-1}}}$.
        If $\{\lambda_i\}^T_{i=1}$ decay polynomially,
        that is, there is a constant $R_0>0$ and $p\geq 1$,
        such that $\lambda_i\leq R_0i^{-p}$,
        then $\vert S_T\vert\leq \mathrm{e}(\frac{C_2R_0}{\alpha})^{\frac{1}{p}}$.
        In both cases,
        $C_1$ and $C_2$ are constants,
        and $R_0=\Theta(T)$.
    \end{lemma}

    Next we give the regret bound and the computational complexity of AOGD-ALD.

    \begin{theorem}
    \label{thm:ICML2023:AOGD-ALD}
        Let $B_0=\lfloor(\sqrt{d^2+4dT}-d)/2\rfloor$ and $\alpha=T^{-1}$.
        For any $\mathcal{I}_T$ satisfying $T>\ln^2{T}$,
        the regret of AOGD-ALD satisfies,
        $$
            \forall f\in\mathbb{H},\quad\mathrm{Reg}(f)=O\left(U\sqrt{L(f)+U}+U^2\right).
        $$
        If $\{\lambda_i\}^T_{i=1}$ decay exponentially,
        then the computational complexity is $O(d\ln{T}+\ln^2{T})$.
        If $\{\lambda_i\}^T_{i=1}$ decay polynomially with degree $p\geq 1$,
        then the computational complexity is
        $O\left(\min\{dT^{\frac{2}{p}}+T^{\frac{4}{p}},dT\}\right).
        $
    \end{theorem}

    Let $f^\ast=\mathrm{argmin}_{f\in\mathbb{H}}L(f)$.
    $O(\sqrt{L(f^\ast)})$ is called ``small-loss'' bound
    \cite{Orabona2012Beyond,Lykouris2018Small,Lee2020A,Wang2020Adapting,Zhang2022A}.
    The data-dependent bound is never worse than the worst-case bound i.e., $O(\sqrt{T})$.
    If we select a good kernel function such that $L(f^\ast)\ll T$,
    then we can obtain a regret of $o(\sqrt{T})$.
    If $L_T(f^\ast)=0$,
    then we obtain a regret of $O(1)$.

\subsection{Comparison with Previous Results}

    The challenge of obtaining a regret of $O(\sqrt{L(f)})$ is the computational cost.
    KOGD achieves this regret bound
    at a computational complexity in $O(dT)$ \cite{Zinkevich2003Online}.
    With probability at least $1-\delta$,
    FOGD \cite{Lu2016Large} achieves a regret of
    $O(\sqrt{L(f)}+\frac{\sqrt{TL(f)\ln{\frac{1}{\delta}}}}{\sqrt{D}})$
    at a computational complexity in $O(dD)$.
    We can define $D=o(T)$
    which yields a suboptimal regret bound.
    For completeness,
    we reanalyze the regret of FOGD in the Appendix.
    Theorem \ref{thm:ICML2023:AOGD-ALD} shows that
    if the eigenvalues decay exponentially or polynomially with degree $p>4$,
    AOGD-ALD  achieves the optimal regret at a computational complexity in $o(T)$.

    Note that $L(f^\ast)$ depends on $\{({\bf x}_t,y_t)\}^T_{t=1}$,
    while $\mathrm{d}_{\mathrm{eff}}(\mu)$ depends on $\{{\bf x}_t\}^T_{t=1}$.
    In general, they are not comparable.
    Thus it is not intuitive to compare AOGD-ALD with
    Pros-N-KONS \cite{Calandriello2017Efficient} and PKAWV \cite{Jezequel2019Efficient}.
    We just explain that AOGD-ALD provides a new regret-computational cost trade-off.
    Table \ref{tab:ICML2023:comparison_results} shows that
    the computational complexity of Pros-N-KONS can be smaller than AOGD-ALD,
    but its regret bound is worse in the case of $p\leq 2+\sqrt{5}$.
    The computational complexity of PKAWV is always larger than AOGD-ALD,
    but its regret bound may be better for $p > 1$.
    In the case of $L(f^\ast)\ll T$,
    the regret bound of AOGD-ALD is also very small.

\section{Approximating KONS}

    The square loss function is exp-concave.
    Thus second-order algorithms, such as KONS,
    can obtain a regret of $O(\mu+\mathrm{d}_{\mathrm{eff}}(\mu)\ln{T})$.
    In this section,
    we propose a deterministic approximation of KONS,
    named NONS-ALD.

\subsection{Kernelized ONS}

    For simplicity,
    we use the hypothesis space $\mathcal{H}$.
    At the end of round $t$,
    the KONS algorithm \cite{Calandriello2017Second}
    compute $f_{t+1}$ by the following rule,
    \begin{equation}
    \label{eq:ICML2023:updating_KONS}
    \begin{split}
        {\bf A}_t=&{\bf A}_{t-1}+\eta_t\nabla\ell(f_t(\mathbf{x}_t),y_t)(\nabla\ell(f_t(\mathbf{x}_t),y_t))^\top,\\
        f_{t+1}=&f_t-{\bf A}^{-1}_t\nabla\ell(f_t(\mathbf{x}_t),y_t),
    \end{split}
    \end{equation}
    where ${\bf A}_0=\mu{\bf I}$.
    We give the pseudo-code in Algorithm \ref{alg:ICML2023:KONS}.

    \begin{algorithm}[!t]
        \caption{KONS}
        \footnotesize
        \label{alg:ICML2023:KONS}
        \begin{algorithmic}[1]
        \REQUIRE{${\bf A}_0=\mu{\bf I}$, $f_1=0$.}
        \FOR{$t=1,\ldots,T$}
            \STATE Receive ${\bf x}_t$
            \STATE Compute $\hat{y}_t=f_t(\mathbf{x}_t)$
            \STATE Update ${\bf A}_t={\bf A}_{t-1}+\eta_t\nabla\ell(f_t(\mathbf{x}_t),y_t)
            (\nabla\ell(f_t(\mathbf{x}_t),y_t))^\top$
            \STATE Compute $f_{t+1}=f_t-{\bf A}^{-1}_t\nabla\ell(f_t(\mathbf{x}_t),y_t)$
        \ENDFOR
        \end{algorithmic}
    \end{algorithm}

    KONS nearly stores all of the observed examples.
    At any round $t$,
    the computational complexity is $O(dt+t^2)$.
    To reduce the computational complexity,
    a natural idea is to use the ALD condition to maintain $S_t$.
    However,
    such a approach still has a $O(t\cdot\vert S_t\vert)$ computational complexity.
    Next we briefly explain the reason.

    At any round $t$,
    if $\mathrm{ALD}_t$ holds,
    then we can approximate $\kappa({\bf x}_t,\cdot)$
    by ${\bm \Phi}_{S_t}{\bm \beta}^\ast_t$ and $S_t$ keeps unchanged.
    Then we have ${\bf A}_t={\bf A}_{t-1}+\eta_t
    (\ell'(f_t({\bf x}_t),y_t))^2{\bm \Phi}_{S_t}{\bm \beta}^\ast_t({\bm \Phi}_{S_t}{\bm \beta}^\ast_t)^\top$.
    The key is to compute $f_{t+1}({\bf x}_{t+1})$.
    \begin{theorem}
    \label{thm:ICML2023:KONS_implicit_computing}
    Let $g_t=\ell'(f_t({\bf x}_t),y_t)$ and
    $$
        \hat{\phi}({\bf x}_t)
        =\left\{
        \begin{array}{ll}
            \phi({\bf x}_t)=\kappa({\bf x}_t,\cdot)& \mathrm{if}~\mathrm{ALD}_t~\mathrm{does~not~hold},\\
            {\bm \Phi}_{S_t}{\bm \beta}^\ast_t&\mathrm{otherwise}.\\
        \end{array}
        \right.\\
    $$
    Let $\hat{\nabla}_t=g_t\hat{\phi}({\bf x}_t)$ and
    $\hat{{\bm \Phi}}_{t}=(\sqrt{\eta_1}\hat{\nabla}_1,\ldots,\sqrt{\eta_t}\hat{\nabla}_t)$. Then
    \begin{align*}
        f_{t+1}&({\bf x}_{t+1})
        =\frac{-1}{\mu}\sum^{t}_{\tau=1}
        g_\tau\hat{\phi}({\bf x}_\tau)^\top\phi({\bf x}_{t+1})+\\
        &\frac{1}{\mu}\sum^{t}_{\tau=1}
        g_\tau\hat{\phi}({\bf x}_\tau)^\top\hat{{\bm \Phi}}_{\tau}
        (\hat{{\bm \Phi}}^\top_{\tau}\hat{{\bm \Phi}}_{\tau}
        +\mu{\bf I})^{-1}\hat{{\bm \Phi}}^\top_{\tau}\phi(\mathbf{x}_{t+1}).
    \end{align*}
    \end{theorem}

    Computing the first term requires time in $O(d\vert S_t\vert)$.
    Computing the second term requires time in $O(t\vert S_t\vert)$.
    The computational challenge comes from that
    KONS runs in the implicit feature space $\mathbb{R}^n$ in which
    we can not explicitly store and incrementally update ${\bf A}_t$.
    To address this issue,
    we use the Nystr\"{o}m projection to approximate the kernel mapping $\phi(\cdot)$,
    and run online Newton step (ONS) in an explicit feature space.
    To be specific,
    let $\phi_j(\cdot):\mathcal{X}\rightarrow \mathbb{R}^j$
    be an approximate kernel mapping.
    Then ${\bf A}_t={\bf A}_{t-1}+\eta_tg^2_t\phi_j({\bf x}_t)\phi^\top_j({\bf x}_t)$.
    We only store ${\bf A}_t\in\mathbb{R}^{j\times j}$, $j<\infty$
    and can incrementally update ${\bf A}_t$.
    The computational complexity is $O(j^2)$.

\subsection{Nystr\"{o}m Projection}

    We briefly introduce how the Nystr\"{o}m projection constructs explicit feature mapping
    \cite{Williams2001Using}.

    We select $j$ columns from ${\bf K}_T$ to form a matrix $\mathbf{K}_{T,j}\in\mathbb{R}^{T\times j}$,
    and select the corresponding $j$ rows from ${\bf K}_T$
    to form a matrix $\mathbf{K}^\top_{T,j}\in\mathbb{R}^{j\times T}$.
    Let $S(j)$ contain the selected instances and
    ${\bf K}_{S(j)}\in\mathbb{R}^{j\times j}$ be the crossing matrix
    whose SVD is ${\bf K}_{S(j)}={\bf U}_{S(j)}\Sigma_{S(j)}{\bf U}^\top_{S(j)}$.
    The Nystr\"{o}m projection approximates ${\bf K}_T$ by
    \begin{align*}
        {\bf K}_T
        \approx&\mathbf{K}_{T,j}{\bf K}^{+}_{S(j)}\mathbf{K}^\top_{T,j}\\
        =&(\Sigma^{-\frac{1}{2}}_{S(j)}{\bf U}^\top_{S(j)}{\bm \Phi}^\top_{S(j)}{\bm \Phi}_T)^\top
        \Sigma^{-\frac{1}{2}}_{S(j)}{\bf U}^\top_{S(j)}
        {\bm \Phi}^\top_{S(j)}{\bm \Phi}_T.
    \end{align*}
    ${\bm \Phi}_T=(\phi({\bf x}_t)_{t\in [T]})\in\mathbb{R}^{n\times T}$
    and ${\bm \Phi}_{S(j)}=(\phi({\bf x})_{{\bf x}\in S(j)})\in\mathbb{R}^{n\times j}$.
    Denote by $\mathcal{P}_{S(j)}={\bm \Phi}_{S(j)}{\bf U}_{S(j)}\Sigma^{-1}_{S(j)}{\bf U}^\top_{S(j)}
    {\bm \Phi}^\top_{S(j)}$
    the projection matrix onto the column space of ${\bm \Phi}_{S(j)}$.
    The approximate scheme defines an explicit feature mapping
    $$
        {\phi}_j(\cdot):\mathcal{X}\rightarrow
        \mathcal{P}^{\frac{1}{2}}_{S(j)}{\phi}(\cdot)\in\mathbb{R}^j,
    $$
    in which $\mathcal{P}^{\frac{1}{2}}_{S(j)}=
        \Sigma^{-\frac{1}{2}}_{S(j)}{\bf U}^\top_{S(j)}{\bm \Phi}^\top_{S(j)}$.
    It is obvious that
    the approximation error depends on the selected $j$ columns, or the crossing matrix.
    In the next subsection,
    we will use the ALD condition to select columns.

\subsection{Column Selecting by the ALD Condition}

    At the beginning of the $t$-th round,
    assuming that $\vert S_t\vert =j$.
    Denote by $S_t=S(j)$.
    We first decide whether ${\bf x}_t$ will be added into $S(j)$.
    Solving \eqref{eq:AAAI23:projection_ALD}, we obtain
    $$
        {\bm \beta}^\ast_j(t) = {\bf K}^{-1}_{S(j)}{\bm \Phi}^\top_{S(j)}\phi({\bf x}_t).
    $$
    If the $\mathrm{ALD}_t$ condition holds,
    that is
    $$
        \alpha_t=\left\Vert {\bm \Phi}_{S(j)}{\bm \beta}^\ast_j(t)
        -\kappa(\mathbf{x}_t,\cdot)\right\Vert^2_{\mathcal{H}}\leq \alpha,
    $$
    then $S(j)$ keeps unchanged.
    We use ${\bf K}_{S(j)}$ as the crossing matrix
    and defined
    $\phi_{j}(\mathbf{x}_t)
        =\mathcal{P}^{\frac{1}{2}}_{S(j)}\phi(\mathbf{x}_{t})$.
    If the $\mathrm{ALD}_t$ condition does not hold,
    then we execute $S_{t+1}=S_t\cup\{({\bf x}_t,y_t)\}$,
    and denote by $S_{t+1}=S(j+1)$.
    We will construct explicit feature $\phi_{j+1}({\bf x}_t)$.

\subsection{Algorithm}

    Let $\mathbb{W}_t=\{f\in\mathcal{H}:\vert f({\bf x}_t)\vert\leq U\}$ \citep{Luo2016Efficient,Calandriello2017Efficient}.
    For each $f\in\mathbb{H}$,
    $\vert f({\bf x}_t)\vert\leq \Vert f\Vert_{\mathcal{H}}\Vert\phi({\bf x}_t)\Vert_{\mathcal{H}}\leq U$.
    Thus $\mathbb{H}\subseteq \mathbb{W}_t$.
    Our algorithm will run in $\{\mathbb{W}_t\}^T_{t=1}$ not $\mathbb{H}$,
    since projection onto $\mathbb{W}_t$ is computationally more efficient.

    We divide the time horizon $\{1,\ldots,T\}$ into different epochs.
    \begin{align*}
        T_0&=\left\{s_1,\ldots,s_j\ldots,s_J:\mathrm{ALD}_{s_j}~\mathrm{does~not~hold}\right\},\\
        T_j&=\{s_j,s_j+1,\ldots,s_{j+1}-1\},~j=1,2,\ldots,J,
    \end{align*}
    where we define $s_1=1$ and $s_{J+1}-1=T$.
    Thus $\{1,\ldots,T\}=\cup^J_{j=1}T_j$.
    For any $t\in T_j$, let $S_t=S(j)=\{{\bf x}_{s_1},\ldots,{\bf x}_{s_j}\}$.
    $\forall j\in [J], t\in T_j\setminus \{s_j\}$,
            the $\mathrm{ALD}_t$ condition holds.
    Besides,
    it is obvious that $\kappa({\bf x}_{s_j},\cdot)\in {\bm \Phi}_{S(j)}$.

    The main idea of our algorithm is to run ONS on $T_j$, $j\in[J]$.
    Next we consider a fixed epoch $T_j$.
    At the beginning of round $t$,
    we compute $\phi_j({\bf x}_t)$.
    Our algorithm maintains a linear hypothesis
    $f_{j,t}(\cdot)={\bf w}^\top_j(t)\phi_j(\cdot)$,
    where ${\bf w}_j(t)\in\mathbb{R}^j$.
    The prediction is given by $\hat{y}_t=f_{j,t}(\mathbf{x}_t)$.
    For simplicity,
    let $g_{j}(t)=2(f_{t,j}({\bf x}_t)-y_t)$, and
    $$
        \nabla_{j}(t)=\nabla\ell(f_{j,t}({\bf x}_t),y_t)=g_{j}(t)\phi_{j}({\bf x}_t).
    $$
    We execute the following updating
    $$
        \left\{
        \begin{array}{ll}
        {\bf A}_j(t)={\bf A}_j(t-1)+\eta_tg^2_{j}(t)\phi_{j}({\bf x}_t)\phi^\top_{j}({\bf x}_t),\\
        \tilde{{\bf w}}_j(t+1)={\bf w}_j(t)-{\bf A}^{-1}_j(t)\nabla_{j}(t)\in\mathbb{R}^{j},\\
        {\bf w}_j(t+1)=\mathcal{P}_{\mathbb{W}_{t+1}}(\tilde{{\bf w}}_j(t+1)),
        \end{array}
        \right.
    $$
    where $\mathcal{P}_{\mathbb{W}_{t+1}}(\cdot)$ is a projection operator defined as follows
    \begin{equation}
    \label{eq:ICML2023:projection}
        {\bf w}_j(t+1)=\mathop{\arg\min}_{{\bf w}\in\mathbb{W}_{t+1}}
        \Vert {\bf w}-\tilde{{\bf w}}_j(t+1)\Vert^2_{{\bf A}_j(t)}.
    \end{equation}
    The initial configurations are denoted by ${\bf A}_j(s_j-1)$ and ${\bf w}_j(s_j)$.
    When we enter $T_j$ from $T_{j-1}$,
    the dimension of explicite feature mapping changes from $j-1$ to $j$
    which induces a technical challenge on initializing the configurations.
    To be specific,
    we can not use $f_{j-1,s_j}={\bf w}^\top_{j-1}(s_j)\phi_{j-1}(\cdot)$
    to prediction ${\bf x}_{t}, t\in T_{j}$.
    To address this issue,
    a simple approach is the restart technique.
    We just need to run a new ONS in $T_{j}$,
    which implies ${\bf A}_j(s_j-1)=\alpha{\bf I}$ and ${\bf w}_j(s_j)={\bf 0}$.
    This idea is adopted by PROS-N-KONS \citep{Calandriello2017Efficient}.
    The simple restart technique increases the regret by a factor of $O(J)$.
    Intuitively,
    the restart technique discards all of the information contained in
    ${\bf w}_{j-1}(s_j)\in\mathbb{R}^{j-1}$ and ${\bf A}_r(s_{r+1}-1)
    \in\mathbb{R}^{r\times r}$, $r\leq j-1$.
    Next we redefine the initial configurations.
    The main idea is to project ${\bf w}_{j-1}(s_j)$ onto $\mathbb{R}^j$
    and project ${\bf A}_r(s_{r+1}-1)$ onto $\mathbb{R}^{j\times j}$, $r\leq j-1$.

    The definition of ${\bf A}_{j}(s_j-1)$ is intuitive.
    For any $t\in T_j$,
    the updating rule of ONS is as follows,
    $$
        {\bf A}_{j}(t)={\bf A}_{j}(s_j-1)+\sum^t_{\tau=s_j}\eta_{\tau}
        g^2_j(\tau)\phi_j({\bf x}_\tau)\phi^\top_j({\bf x}_\tau).
    $$
    The ideal value of ${\bf A}_j(s_j-1)$ should be
    $$
        {\bf A}_j(s_j-1)
        =\mu{\bf I}+\sum^{j-1}_{r=1}\sum_{t\in T_r}\eta_tg^2_r(t)\phi_{j}({\bf x}_t)\phi^\top_{j}({\bf x}_t),
    $$
    where $\phi_{j}({\bf x}_t)=\mathcal{P}^{\frac{1}{2}}_{S(j)}\phi({\bf x}_t)$.
    However,
    such an approach must store $\{{\bf x}_t\}^{s_j-1}_{t=1}$
    which induce a $O(dT)$ computational complexity.
    Recalling that the ALD condition guarantees that
    $\phi({\bf x}_t)\approx{\bm \Phi}_{S(r)}{\bm\beta}^\ast_r(t)$.
    It is natural to define
    \begin{equation}
    \label{eq:ICML2023:tilde_phi}
        \forall t\in T_r,\quad
        \tilde{\phi}_j({\bf x}_t)=\mathcal{P}^{\frac{1}{2}}_{S(j)}{\bm \Phi}_{S(r)}{\bm \beta}^\ast_r(t).
    \end{equation}
    We can define ${\bf A}_j(s_j-1)$ as follows,
    \begin{equation}
    \label{eq:ICML2023:approximate_A_t}
        {\bf A}_j(s_j-1)
        =\mu{\bf I}+\sum^{j-1}_{r=1}\sum_{t\in T_r}\eta_tg^2_r(t)\tilde{\phi}_{j}({\bf x}_t)
        \tilde{\phi}^\top_{j}({\bf x}_t).
    \end{equation}
    In this way,
    we only use the instances in $S_t$.
    The computational complexity is $O(d\vert S_t\vert+\vert S_t\vert^2)$.

    It is less intuitive to define ${\bf w}_j(s_j)$.
    The projection of any $f\in\mathbb{H}$ onto the column space of ${\bm \Phi}_{S(j-1)}$
    and ${\bm \Phi}_{S(j)}$ are $f_{j-1}=\mathcal{P}_{S(j-1)}f$ and $f_j=\mathcal{P}_{S(j)}f$, respectively.
    Denote by $f_{j-1}={\bf w}^\top_{j-1}\phi_{j-1}(\cdot)$ and $f_j={\bf w}^\top_j\phi_j(\cdot)$.
    We can prove that
    ${\bf w}_{j-1}=\mathcal{P}^\frac{1}{2}_{S(j-1)}(\mathcal{P}^{\frac{1}{2}}_{S(j)})^\top{\bf w}_j$.
    Thus it must be
    \begin{equation}
    \label{eq:ICML2023:approximate_w_t:property_1}
        {\bf w}_{j-1}(s_j)
        =\mathcal{P}^\frac{1}{2}_{S(j-1)}(\mathcal{P}^{\frac{1}{2}}_{S(j)})^\top{\bf w}_{j}(s_j).
    \end{equation}
    Besides, at the $(s_j-1)$-th round,
    ${\bf w}_{j-1}(s_j)$ must be the solution of the following projection
    \begin{equation}
    \label{eq:ICML2023:approximate_w_t:property_2}
        {\bf w}_{j-1}(s_j)=\mathcal{P}_{\mathbb{W}_{s_j}}(\tilde{{\bf w}}_{j-1}(s_j)).
    \end{equation}
    To this end,
    we need to compute $\phi_{j-1}({\bf x}_{s_j})$.
    Note that $\mathrm{ALD}_{s_j}$ does not hold.
    Although
    $\kappa({\bf x}_{s_j},\cdot)$ can not be well approximated by ${\bm \Phi}_{S(j-1)}$,
    the goal of
    \eqref{eq:ICML2023:approximate_w_t:property_2} is just to ensure
    ${\bf w}_{j}(s_j)\in \mathbb{W}_{s_j}$.
    Both the property in \eqref{eq:ICML2023:approximate_w_t:property_1}
    and \eqref{eq:ICML2023:approximate_w_t:property_2} are critical to
    the regret analysis.

    We name this algorithm NONS-ALD (Nystr\"{o}m Online Newton Step using the ALD condition),
    and give the pseudo-code in Algorithm \ref{alg:ICML2023:NONS-ALD}.

    \begin{algorithm}[!t]
        \caption{\small{NONS-ALD}}
        \footnotesize
        \label{alg:ICML2023:NONS-ALD}
        \begin{algorithmic}[1]
        \REQUIRE{$\mu$, $\alpha$, $U$, $Y$}
        \ENSURE{$j=0$, ${\bf w}_1(1)=0, {\bf A}_1(0)=\mu$, $S(0)=\emptyset$, $\mathrm{flag}=1$}
        \FOR{$t=1,\ldots,T$}
            \STATE Receive $\mathbf{x}_t$
            \STATE Compute ${\bm \beta}^\ast_j(t)=\arg\min_{{\bm \beta}\in\mathbb{R}^j}
                  \Vert\phi(\mathbf{x}_t)-{\bm \Phi}_{S(j)}{\bm \beta}\Vert^2_{\mathcal{H}}$
            \STATE Compute $\alpha_t=\kappa(\mathbf{x}_t,\mathbf{x}_t)
            -\phi(\mathbf{x}_t)^\top{\bm \Phi}_{S(j)}{\bm \beta}^\ast_j(t)$
            \IF{$\alpha_t> \alpha$}
                \STATE $S(j+1)=S(j)\cup\{({\bf x}_t,y_t)\}$
                \STATE $\mathrm{flag}=1$
                \STATE $j=j+1$
            \ELSE
                \IF{$\mathrm{flag}==1$}
                    \STATE $s_j=t$
                    \STATE $({\bf U}_{S(j)},\Sigma_{S(j)})\leftarrow \mathrm{SVD}({\bf K}_{S(j)})$
                    \STATE $\mathrm{flag}=0$
                    \STATE Compute $Q_{j,j-1}=\mathcal{P}^{\frac{1}{2}}_{S(j)}
                    (\mathcal{P}^{\frac{1}{2}}_{S(j-1)})^\top$
                    \STATE Compute ${\bf A}_j(s_j-1)$ follows Lemma \ref{prop:ICML2023:approximate_A_t}
                    \STATE Compute ${\bf w}_j(s_j)
                =\mathcal{P}^{\frac{1}{2}}_{S(j)}(\mathcal{P}^\frac{1}{2}_{S(j-1)})^\top{\bf w}_{j-1}(s_j)$
                \ENDIF
                \STATE Compute $\phi_j({\bf x}_t)=\Sigma^{-\frac{1}{2}}_{S(j)}{\bf U}^\top_{S(j)}
                                {\bf \Phi}^\top_{S(j)}\phi(\mathbf{x}_{t})$
                \STATE Output $\hat{y}_t={\bf w}^\top_j(t)\phi_j({\bf x}_t)$
                \STATE Compute $\nabla_j(t)=\ell'(\hat{y}_t,y_t)\cdot \phi_j({\bf x}_t)$
                \STATE Update ${\bf A}_j(t)={\bf A}_j(t-1)+\eta_t\nabla_j(t)\nabla^\top_j(t)$
                \STATE Compute $\tilde{{\bf w}}_j(t+1)={\bf w}_j(t)-{\bf A}^{-1}_j(t)\nabla_j(t)$
                \STATE Compute $\phi_j({\bf x}_{t+1})=\Sigma^{-\frac{1}{2}}_{S(j)}{\bf U}^\top_{S(j)}
                                {\bf \Phi}^\top_{S(j)}\phi(\mathbf{x}_{t+1})$
                \STATE Compute ${\bf w}_j(t+1)$ following \eqref{eq:ICML2023:projection}
            \ENDIF
        \ENDFOR
        \end{algorithmic}
    \end{algorithm}

\subsection{Theoretical Analysis}

\subsubsection{Regret analysis}

    We first show an equivalent definition of \eqref{eq:ICML2023:approximate_A_t}.
    \begin{lemma}
    \label{prop:ICML2023:approximate_A_t}
        For any $j=1,\ldots,J$,
        the approximate scheme \eqref{eq:ICML2023:approximate_A_t} is equivalent to the following scheme
        $$
            {\bf A}_j(s_j-1)= \mu{\bf I}+Q_{j,j-1}({\bf A}_{j-1}(s_{j}-1)-\mu{\bf I})Q^\top_{j,j-1}
        $$
        where $Q_{j,j-1}=\mathcal{P}^{\frac{1}{2}}_{S(j)}(\mathcal{P}^{\frac{1}{2}}_{S(j-1)})^\top$.
    \end{lemma}

    Storing ${\bf A}_j(s_j-1)$ and $Q_{j,j-1}$ requires space in $O(j^2)$,
    and computing ${\bf A}_j(s_j-1)$ requires time in $O(j^3)$.

    \begin{lemma}
    \label{prop:ICML2023:approximate_w_t}
        For any $j=1,\ldots,J$,
        let ${\bf w}_{j-1}(s_j)$ satisfy \eqref{eq:ICML2023:approximate_w_t:property_2}, and
        $$
            {\bf w}_j(s_j)
            =\mathcal{P}^{\frac{1}{2}}_{S(j)}(\mathcal{P}^\frac{1}{2}_{S(j-1)})^\top{\bf w}_{j-1}(s_j).
        $$
        Then ${\bf w}_j(s_j)\in\mathbb{W}_{s_j}$
        and \eqref{eq:ICML2023:approximate_w_t:property_1} is satisfied.
    \end{lemma}

    \begin{remark}
        An empirical version of Pros-N-KONS \citep{Calandriello2017Efficient},
        named CON-KNOS,
        uses a different ${\bf w}_j(s_j)$.
        CON-KNOS uses ${\bf w}_{j-1}(s_j-1)$ to construct ${\bf w}_j(s_j)$,
        while our algorithm uses ${\bf w}_{j-1}(s_j)$ to construct ${\bf w}_j(s_j)$.
        Our regret analysis shows that
        ${\bf w}_{j-1}(s_j)$ is necessary for obtaining the nearly optimal regret bound.
    \end{remark}

    Next we measure the quality of columns selected by the ALD condition
    using spectral norm error bounds.

    \begin{lemma}[Spectral Norm Error Bound]
    \label{lemma:ICML2023:spectral_norm_error_kernel_approximate}
        Let $\alpha\leq 1$.
        For all $j=1,\ldots,J$,
        let ${\bm \Phi}_{T_j}=(\phi({\bf x}_t))_{t\in T_j}$
        and $\mathcal{P}_{S(j)}$ be the projection matrix onto the column space of ${\bm \Phi}_{S(j)}$.
        \begin{equation}
        \label{eq:ICML2023:local_spectral_norm_error_kernel_approximate}
            \forall j\in[J],~\left\Vert {\bm \Phi}^\top_{T_j}{\bm \Phi}_{T_j}
            -{\bm \Phi}^\top_{T_j}\mathcal{P}_{S(j)}{\bm \Phi}_{T_j}\right\Vert_2 \leq \vert T_j\vert\cdot\alpha.
        \end{equation}
        Let $\tilde{{\bm \Phi}}_T=\left((\tilde{\phi}_J({\bf x}_t))_{t\in T_1},\ldots,
        (\tilde{\phi}_J({\bf x}_t))_{t\in T_J}\right)\in\mathbb{R}^{J\times T}$,
        where $\tilde{\phi}_J(\cdot)$ follows \eqref{eq:ICML2023:tilde_phi}.
        Then
        \begin{equation}
        \label{eq:ICML2023:spectral_norm_error_kernel_approximate}
            \left\Vert {\bf K}_T
            -\tilde{{\bm \Phi}}^\top_T\tilde{{\bm \Phi}}_T\right\Vert_2\leq T\sqrt{\alpha}.
        \end{equation}
    \end{lemma}

    We call \eqref{eq:ICML2023:spectral_norm_error_kernel_approximate}
    global spectral norm error bound.
    We call \eqref{eq:ICML2023:local_spectral_norm_error_kernel_approximate}
    local spectral norm error bound.
    According to Lemma \ref{lemma:ICML2023:spectral_norm_error_kernel_approximate},
    we can prove that the regret induced by our projection scheme
    (i.e., projecting ${\bf A}_{j-1}(s_{j}-1)$ and ${\bf w}_{j-1}(s_j)$)
    is controlled by the parameter $\alpha$.
    Thus optimizing $\alpha$ will yield the desired regret bounds.

    Lemma \ref{lemma:ICML2023:spectral_norm_error_kernel_approximate}
    gives deterministic spectral norm error bounds,
    while most of previous results only hold in a high probability,
    such as the uniform column sampling \cite{Drineas2005On,Jin2013Improved}
    and the RSL sampling \cite{Calandriello2017Efficient}.
    If the instances could be observed beforehand,
    such as offline learning,
    then we can obtain a goal spectral norm error bound stated in
    \eqref{eq:ICML2023:local_spectral_norm_error_kernel_approximate}.
    Such a result might be of independent interest.
    In this case,
    previous work only proved
    a global spectral norm error bound of $O(T\sqrt{\alpha})$ \cite{Sun2012On}.

    \begin{theorem}
    \label{thm:ICML2023:NONS-ALD:regret_bound}
        Let $U\geq Y$ and $\eta_t=\frac{1}{4(U^2+Y^2)}$ for all $t\in[T]$.
        Assuming that $\vert S_T\vert = J$.
        For any $f\in\mathbb{H}$,
        the regret of NONS-ALD satisfies
        \begin{align*}
            \mathrm{Reg}(f)
            \leq& (\frac{\mu}{2}+T\alpha)\Vert f\Vert^2_{\mathcal{H}}
            +\frac{1}{2}\mathrm{d}_{\mathrm{eff}}(\frac{\mu}{2})\left(1+\ln \frac{2T+\mu}{\mu}\right)\\
            &+\frac{T^2\sqrt{\alpha}}{\sqrt{2}\mu}
            +\sqrt{8(U^2+Y^2)}\Vert f\Vert_{\mathcal{H}} \cdot T\sqrt{\alpha}.
        \end{align*}
        The space complexity is $O(dJ+J^2)$.
        The average per-round time complexity is $O(dJ+J^2+\frac{J^4}{T})$.
    \end{theorem}

    We will omit the factor $O(dJ)$ in the discussion on computational complexity.
    Next we give the values of $\mu$ and $\alpha$ and derive nearly optimal regret bounds.
    \begin{corollary}
    \label{coro:ICML2023:NONS-ALD:exponential_decay}
        Let $\alpha= \frac{\ln^4{T}}{T^4}$ and $\mu>0$ be a constant.
        If $\{\lambda_i\}^T_{i=1}$ decay exponentially, i.e., $\lambda_i\leq R_0r^{i}$,
        $R_0=\Theta(T)$,
        then the regret of NONS-ALD satisfies
        $$
           \forall f\in\mathbb{H},~\mathrm{Reg}(f)= O(\Vert f\Vert^2_{\mathcal{H}}+\ln^2{T}).
        $$
        The space and average per-round time complexity is $O(\ln^2{T})$.
    \end{corollary}

    If $\{\lambda_i\}^T_{i=1}$ decay polynomially,
    then we must tune $\mu$ and $\alpha$.
    \begin{corollary}
    \label{coro:ICML2023:NONS-ALD:polynomially_decay}
        If $\{\lambda_i\}^T_{i=1}$ decay polynomially with degree $p\geq 1$,
        i.e., $\lambda_i\leq R_0i^{-p}$,
        $R_0=\Theta(T)$,
        then let $\mu=T^{\frac{1}{1+p}}$ and $\alpha=T^{-\frac{4p}{1+p}}$.
        The regret of NONS-ALD satisfies
        $$
            \forall f\in\mathbb{H},~\mathrm{Reg}(f)=O\left(T^{\frac{1}{1+p}}\ln{T}\right).
        $$
        The space complexity is $O(T^{\frac{2(1+5p)}{p(1+p)}})$,
        and the average per-round time complexity is
        $O(T^{\frac{2(1+5p)}{p(1+p)}}+T^{\frac{4(1+5p)}{p(1+p)}-1})$.
    \end{corollary}

    It is worth mentioning that
    for all $p\geq 10$,
    the space complexity and
    the average per-round time complexity is $O(T^{\frac{2(1+5p)}{p(1+p)}})=o(T)$.
    This is the first algorithm that achieves a nearly optimal regret bound
    at a sublinear computational complexity.
    However,
    the computational complexity becomes worse for $p<10$.
    It is left to further work to
    achieve the same regret bound at a $o(T)$ computational complexity in the case of $p<10$.

    The regret bounds in
    Corollary \ref{coro:ICML2023:NONS-ALD:exponential_decay}
    and
    Corollary \ref{coro:ICML2023:NONS-ALD:polynomially_decay}
    recovery the regret bounds of KONS \cite{Calandriello2017Second}.
    Our regret bounds are optimal up to $\ln{T}$.
    The most important improvement is the computational complexity.
    KONS requires a $O(T^2)$ computational complexity.

\subsubsection{Comparison with more results}

    We compare our algorithm with Pros-N-KONS \cite{Calandriello2017Efficient}
    and PKAWV \cite{Jezequel2019Efficient}.

    With probability at least $1-\delta$,
    Pros-N-KONS achieves
    $$
        \forall f\in\mathbb{H},\mathrm{Reg}(f)\leq\frac{\mu}{2}J\Vert f\Vert^2_{\mathcal{H}}+J\cdot\mathrm{d}_{\mathrm{eff}}(\mu)\ln(T)+
            \frac{T\alpha}{\mu},
    $$
    where $J=O(\mathrm{d}_{\mathrm{eff}}(\alpha)\cdot\ln^2\frac{T}{\delta})$.
    The space complexity is $O(J^2)$.
    Pros-N-KONS executes the SVD operations $J$ times.
    Thus the average per-round time complexity is $O(J^2+\frac{J^4}{T})$.
    The factor $O(\ln^2{\frac{T}{\delta}})$ on $J$
    is induced by the RLS sampling (see Proposition 1 in \citet{Calandriello2017Efficient})
    which is a random method.
    Thus $O(\ln^2{\frac{T}{\delta}})$ is unavoidable.
    Our algorithm uses the ALD condition which is a deterministic method,
    and does not have the factor.

    If $\{\lambda_i\}^T_{i=1}$ decay exponentially,
    then $\mathrm{d}_{\mathrm{eff}}(\alpha)=O(\ln{\frac{T}{\alpha}})$.
    Let $\mu$ be a constant and $\alpha=\frac{\mu}{T}$.
    Pros-N-KONS enjoys a regret of $O(\ln^5{T})$
    at a computational complexity (space complexity and average time-complexity) in $O(\ln^6{T})$.
    Our algorithm enjoys a regret of $O(\ln^2{T})$
    at a computational complexity in $O(\ln^2{T})$.

    If $\{\lambda_i\}^T_{i=1}$ decay polynomially,
    then $\mathrm{d}_{\mathrm{eff}}(\alpha)=O((\frac{T}{\alpha})^{\frac{1}{p}})$.
    We solve the following two equations
    \begin{align*}
        \mu=\left(\frac{T}{\mu}\right)^{\frac{1}{p}},\quad
        \mu\left(\frac{T}{\alpha}\right)^{\frac{1}{p}} =\frac{T\alpha}{\mu}.
    \end{align*}
    The solutions are
    $\mu=T^{\frac{1}{1+p}}$ and $\alpha=T^{-\frac{p^2-2p-1}{(p+1)^2}}$.
    Pros-N-KONS enjoys a regret of $O(T^{\frac{3p+1}{(1+p)^2}}\ln^3{T})$
    at a computational complexity in
    $O(T^{\frac{4p}{(1+p)^2}}\ln^4{T})$.
    Although Pros-N-KONS ensures a computational complexity in $o(T)$ for $p>1$,
    its regret bound is far from optimal.

    With probability at least $1-\delta$,
    PKAWV achieves
    $$
        \forall f\in\mathbb{H},\quad\mathrm{Reg}(f)\leq\frac{\mu}{2}\Vert f\Vert^2_{\mathcal{H}}+\mathrm{d}_{\mathrm{eff}}(\mu)\ln(T)+
            \frac{JT\alpha}{\mu}.
    $$
    The computational complexity is $O(TJ+J^2)$,
    where $J=O(\mathrm{d}_{\mathrm{eff}}(\alpha)\cdot\ln^2\frac{T}{\delta})$
    (see Algorithm 2 in the Supplementary material of \citet{Jezequel2019Efficient},
    or see Section H in \citet{Jezequel2019EfficientCoRR}).
    Besides, at each round $t$,
    PKAWV must store the pervious examples $\{({\bf x}_\tau,y_\tau)^t_{\tau=1}\}$.
    Both our algorithm and Pros-N-KONS only store $J$ examples.

    If $\{\lambda_i\}^T_{i=1}$ decay exponentially,
    then PKAWV enjoys a regret of $O(\ln^2{T})$
    at a computational complexity in $O(T\ln^3{T})$.
    Our algorithm enjoys the same regret bound only at a computational complexity in $O(\ln^2{T})$.

    If $\{\lambda_i\}^T_{i=1}$ decay polynomially,
    then PKAWV also enjoys a regret of $O(T^{\frac{1}{1+p}}\ln{T})$.
    PKAWV suffers a computational complexity in
    $O(T^{\frac{4p}{p^2-1}}\ln^2{T}+T^{1+\frac{2p}{p^2-1}}\ln^4{T})$
    which can not be $o(T)$ for all $p\geq 1$.
    In the case of $p\geq 10$,
    our algorithm enjoys a computational complexity in $o(T)$.

    Finally,
    we note that Pros-N-KONS can compare with
    $f\in\mathbf{H}=\{f\in\mathcal{H}:\forall t\in[T],\vert f({\bf x}_t)\vert\leq U\}$
    and PKAWV can compare with $f\in\mathcal{H}$,
    while
    our algorithm only compares with $f\in\mathbb{H}$.
    It should be that $\mathbb{H}\subseteq \mathbf{H}\subseteq\mathcal{H}$.
    From the perspective of the size of hypothesis space,
    our algorithm is weaker than Pros-N-KONS and PKAWV.
    As explained in Section \ref{sec:ICML2023:problem_definition_OKR},
    it is enough to compare with hypotheses in $\mathbb{H}$.

\subsubsection{Computational Complexity analysis}

    At each round $t$,
    the main time cost is to compute
    the projection \eqref{eq:ICML2023:projection},
    ${\bf A}_j(s_j-1)$, ${\bf A}^{-1}_j(s_j)$ and the SVD of ${\bf K}_{S(j)}$.

    The solution of projection \eqref{eq:ICML2023:projection} is as follows.
    \begin{theorem}[\citet{Luo2016Efficient}]
    At each round $t$,
    $$
        {\bf w}_j(t+1)=\tilde{{\bf w}}_j(t+1)
        -\frac{m(\tilde{y}_{t+1}){\bf A}^{-1}_j(t)\phi_j({\bf x}_{t+1})}
        {\phi^\top_j({\bf x}_{t+1}){\bf A}^{-1}_j(t)\phi_j({\bf x}_{t+1})}
    $$
    where $\tilde{y}_{t+1}=\tilde{{\bf w}}^\top_j(t+1)\phi_j({\bf x}_{t+1})$
    and $m(\tilde{y}_{t+1})=\mathrm{sign}(\tilde{y}_{t+1})\max\{\vert \tilde{y}_{t+1}\vert-U,0\}$.
    \end{theorem}
    For any invertible ${\bf B}\in\mathbb{R}^{j\times j}$ and ${\bf a},{\bf b}\in\mathbb{R}^j$,
    we have
    $$
        ({\bf B}+{\bf a}{\bf b}^\top)^{-1}
        = {\bf B}^{-1} - \frac{{\bf B}^{-1}{\bf a}{\bf b}^{\top}{\bf B}^{-1}}{1+{\bf b}^\top {\bf B}^{-1}{\bf a}}.
    $$
    Let ${\bf B}={\bf A}_j(t-1)$ and ${\bf a}={\bf b}=\sqrt{\eta_t}\nabla_j(t)$.
    In this way,
    ${\bf A}^{-1}_j(t))$ can be computed incrementally in time $O(j^2)$.
    The time complexity over $T$ rounds is $O(TJ^2)$.

    Computing
    ${\bf A}_j(s_j-1)$, ${\bf A}^{-1}_j(s_j-1)$ and the SVD of ${\bf K}_{S(j)}$
    requires time in $O(j^3)$.
    Such operations are only executed $J$ times.
    The time complexity over $T$ rounds is $O(J^4)$.

    Thus the total complexity is $O(TJ^2+J^4)$.
    Thus the average per-round time complexity is $O(J^2+\frac{J^4}{T})$.
    The space complexity is always $O(J^2)$.

\section{Conclusion}
    In this paper,
    we have studied the trade-off
    between regret and computational cost for online kernel regression,
    and proposed two algorithms that
    achieve two types of nearly optimal regret bounds
    at a sublinear computational complexity for the first time.
    The two regret bounds are data-dependent and not comparable.
    The computational complexities of our algorithms depend on
    the decay rate of eigenvalues of the kernel matrix,
    and are sublinear if the eigenvalues decay fast enough.
    We empirically verified that
    our algorithms can balance the prediction performance and computational cost better
    than previous algorithms can do.

    The two algorithms use the ALD condition to
    dynamically maintain a group of nearly orthogonal basis
    which are used to approximate the kernel mapping.
    Compared with other basis selecting schemes,
    such as uniform sampling and the RLS sampling,
    both the number of basis and the approximate error bound can be smaller.
    The ALD condition can be a better basis selecting scheme
    for designing computationally efficient online and offline kernel learning algorithms.

\section*{Acknowledgements}

    This work is supported by
    the National Natural Science Foundation of China under grants No. 62076181.
    We thank all anonymous reviewers for their valuable comments and suggestions.

\nocite{langley00}

\bibliography{OKR}
\bibliographystyle{icml2023}

\newpage
\appendix
\onecolumn
\section{Experiments}
\label{sec:ICML2023:Experiments}

    In this section, we verify the following three goals.
    \begin{enumerate}[\textbf{G} 1]
      \item NONS-ALD enjoys the best prediction performance.\\
            Corollary \ref{coro:ICML2023:NONS-ALD:exponential_decay}
            and Corollary \ref{coro:ICML2023:NONS-ALD:polynomially_decay}
            show that the regret bounds of NONS-ALD are optimal up to $\ln{T}$.
            We expect that NONS-ALD performs best.
      \item If the kernel function is well tuned,
            then AOGD-ALD and NONS-ALD only store few examples. \\
            The number of stored examples
            depends on the decay rate of eigenvalues of the kernel matrix.
            We can tune the kernel function such that
            the eigenvalues of the kernel matrix decay fast.
      \item If the kernel function is well tuned,
            then AOGD-ALD and NONS-ALD are computationally efficient. \\
            We have proved that
            if the eigenvalues of the kernel matrix decay fast,
            then AOGD-ALD and NONS-ALD achieve a $o(T)$ computational complexity.
            If the eigenvalues decay exponentially,
            the computational complexity is $O(\ln^2{T})$.
    \end{enumerate}

\subsection{Experimental Setting}

    We adopt the Gaussian kernel
    $\kappa(\mathbf{x},\mathbf{v})=\exp(-\frac{\Vert\mathbf{x}-\mathbf{v}\Vert^2}{2\varsigma^2})$
    and use $8$ regression datasets from WEKA
    and UCI machine learning repository
    \footnote{https://archive.ics.uci.edu/ml/index.php}.
    The information of datasets is given in Table \ref{tab:ICML2023:datasets}.
    The target variables and features of all datasets are rescaled to fit in
    $[0,1]$ and $[-1,1]$ respectively.
    We randomly permutate the instances in the datasets 10 times
    and report the average results.
    All algorithms are implemented with R on
    a Windows machine with 2.8 GHz Core(TM) i7-1165G7 CPU
    \footnote{The codes are available at https://github.com/JunfLi-TJU/OKR.git.}.

    \begin{table}[!t]
      \centering
      \setlength{\tabcolsep}{1.6mm}
      \caption{Basic information of datasets. \#num is the number of examples. \#fea is the number of features.}
      \label{tab:ICML2023:datasets}
      \begin{tabular}{lrr|lrr|lrr|lrr}
        \Xhline{0.8pt}
        {Dataset}&\#num & \#fea &Dataset&\#num & \#fea &Dataset&\# num & \#fea &Dataset&\#num & \#fea\\
        \hline
        parkinson  & 5875  & 16  & elevators  & 16599  & 18 & cpusmall & 8192  & 12 & bank     & 8192   & 32\\
        ailerons   & 13750 & 40  & calhousing & 14000  & 8  & Year     & 51630 & 90 & TomsHardware & 28179  & 96\\
        \Xhline{0.8pt}
      \end{tabular}
    \end{table}

    The baseline algorithms include two first-order algorithms,
    FOGD and NOGD \cite{Lu2016Large}
    and two second-order algorithms,
    PROS-N-KONS and CON-KONS \cite{Calandriello2017Efficient}.
    CON-KONS which is an empirical variant of PROS-N-KONS,
    sets ${\bf A}_j(s_j-1)= Q_{j,j-1}{\bf A}_{j-1}(s_{j}-1)Q^\top_{j,j-1}$
    and ${\bf w}_j(s_j)=Q_{j,j-1}{\bf w}_{j-1}(s_j-1)$,
    where $Q_{j,j-1}=\mathcal{P}^{\frac{1}{2}}_{S(j)}(\mathcal{P}^\frac{1}{2}_{S(j-1)})^\top$.
    Note that the two values are different from
    our initial configurations in Lemma \ref{prop:ICML2023:approximate_A_t}
    and Lemma \ref{prop:ICML2023:approximate_w_t}.
    We do not compare with PKAWV \cite{Jezequel2019Efficient},
    since its computational complexity is $O(T)$.
    The experimental results in \cite{Jezequel2019Efficient}
    also verified that PKAWV runs slower than PROS-N-KONS.
    For FOGD and NOGD,
    we tune the stepsize $\eta\in\{\frac{1}{\sqrt{T}},\frac{10}{\sqrt{T}},\frac{100}{\sqrt{T}},\frac{1000}{\sqrt{T}}\}$.
    We set $D=400$ for FOGD and $J=400$ for NOGD
    in which
    $D$ is the number of random features and
    $J$ is the size of buffer (or the number of stored examples).
    There are five hyper-parameters needed to be tuned in PROS-N-KONS and CON-KONS,
    i.e., $C$, $\beta$, $\varepsilon$, $\alpha$ and $\gamma$.
    We set $C=1$, $\beta=1$, $\varepsilon=0.5$ following the suggestion in original paper
    \cite{Calandriello2017Efficient}.
    To improve the performance of PROS-N-KONS and CON-KONS,
    we tune $\alpha\in\{1,5,15\}$ and $\gamma\in\{0.5,1,5,10\}$.
    $\alpha$ is a regularization parameter and
    plays the same role with the parameter $\mu$ in NONS-ALD.
    $\gamma$ controls the size of buffer.
    The larger $\gamma$ is, the smaller the buffer will be,
    that is, the computational complexity will be smaller.
    We set $\alpha=\frac{25}{T}$ for AOGD-ALD and NONS-ALD.
    We set $U=2$ for AOGD-ALD and $U=1$ for NONS-ALD.
    Besides,
    we tune $\mu\in\{1,5,15\}$ for NONS-ALD.

\subsection{Experimental Results}

    \begin{table}[!t]
      \centering
      \setlength{\tabcolsep}{1.6mm}
      \caption{
      Experimental results on benchmark datasets.
      }
      \label{tab:ICML2023:experimental_results}
      \begin{tabular}{l|rrr|rrr}
        \Xhline{0.8pt}
        \multirow{2}{*}{Algorithm}&\multicolumn{3}{c|}{parkinson, $\varsigma=8$}
        &\multicolumn{3}{c}{elevator, $\varsigma=8$}\\
        \cline{2-7}&MSE &$J\vert D$       &Time (s) &MSE &$J\vert D$       &Time (s)       \\
        \hline
        FOGD           & 0.05590 $\pm$ 0.00011  & 400  & 0.26 $\pm$ 0.01
                       & 0.00560 $\pm$ 0.00009  & 400  & 0.72 $\pm$ 0.02\\
        NOGD           & 0.05711 $\pm$ 0.00042  & 400  & 1.42 $\pm$ 0.03
                       & 0.00575 $\pm$ 0.00004  & 400  & 3.97 $\pm$ 0.11\\
        PROS-N-KONS    & 0.06420 $\pm$ 0.00073  & 33   & 0.45 $\pm$ 0.05
                       & 0.00873 $\pm$ 0.00023  & 32   & 1.16 $\pm$ 0.07\\
         CON-KONS      & 0.05553 $\pm$ 0.00024  & 31   & 0.46 $\pm$ 0.04
                       & 0.00452 $\pm$ 0.00018  & 34   & 1.19 $\pm$ 0.14\\
        \hline
        AOGD-ALD       & 0.05988 $\pm$ 0.00018  & 13   & 0.08 $\pm$ 0.02
                       & 0.00534 $\pm$ 0.00003  & 28   & 0.27 $\pm$ 0.02\\
        NONS-ALD       & \textbf{0.05514} $\pm$ \textbf{0.00008}  & \textbf{13}   & 0.14 $\pm$ 0.02
                       & \textbf{0.00284} $\pm$ \textbf{0.00005}  & \textbf{28}   & 0.71 $\pm$ 0.06\\
        \Xhline{0.8pt}
        \multirow{2}{*}{Algorithm}&\multicolumn{3}{c|}{cpusmall , $\varsigma=2$}
        &\multicolumn{3}{c}{bank, $\varsigma=12$}\\
        \cline{2-7}&MSE &$J\vert D$       &Time (s) &MSE &$J\vert D$       &Time (s)       \\
        \hline
        FOGD           & 0.01269 $\pm$ 0.00033  & 400  & 0.34 $\pm$ 0.02
                       & 0.01910 $\pm$ 0.00033  & 400  & 0.43 $\pm$ 0.01\\
        NOGD           & 0.01388 $\pm$ 0.00070  & 400  & 1.93 $\pm$ 0.04
                       & 0.01966 $\pm$ 0.00008  & 400  & 2.09 $\pm$ 0.03\\
        PROS-N-KONS    & 0.02939 $\pm$ 0.00096  & 42   & 0.77 $\pm$ 0.11
                       & 0.02677 $\pm$ 0.00015  & 179  & 14.05 $\pm$ 1.07\\
         CON-KONS      & 0.01166 $\pm$ 0.00080  & 42   & 0.77  $\pm$ 0.08
                       & 0.01663 $\pm$ 0.00014  & 177  & 14.00 $\pm$ 1.15\\
        \hline
        AOGD-ALD       & 0.01330 $\pm$ 0.00006  & 44   & 0.15 $\pm$ 0.02
                       & 0.01915 $\pm$ 0.00009  & 148  & 0.66 $\pm$ 0.03\\
        NONS-ALD       & \textbf{0.00703} $\pm$ \textbf{0.00024}  & \textbf{43}   & 0.62 $\pm$ 0.06
                       & \textbf{0.01306} $\pm$ \textbf{0.00004}  & \textbf{148}   & 11.18 $\pm$ 0.53\\
        \Xhline{0.8pt}
        \multirow{2}{*}{Algorithm}&\multicolumn{3}{c|}{ailerons , $\varsigma=8$}
        &\multicolumn{3}{c}{calhousing, $\varsigma=4$}\\
        \cline{2-7}&MSE &$J\vert D$       &Time (s) &MSE &$J\vert D$       &Time (s)       \\
        \hline
        FOGD           & 0.00363 $\pm$ 0.00009  & 400  & 0.73 $\pm$ 0.02
                       & 0.02690 $\pm$ 0.00017  & 400  & 0.54 $\pm$ 0.01\\
        NOGD           & 0.00394 $\pm$ 0.00013  & 400  & 3.63 $\pm$ 0.06
                       & 0.02800 $\pm$ 0.00032  & 400  & 3.08 $\pm$ 0.05\\
        PROS-N-KONS    & 0.01509 $\pm$ 0.00028  & 88   & 4.45 $\pm$ 0.40
                       & 0.04336 $\pm$ 0.00146  & 45   & 1.52 $\pm$ 0.14\\
         CON-KONS      & 0.00320 $\pm$ 0.00007  & 84   & 4.25 $\pm$ 0.35
                       & 0.02436 $\pm$ 0.00010  & 44   & 1.49 $\pm$ 0.13\\
        \hline
        AOGD-ALD       & 0.00345 $\pm$ 0.00002  & 58   & 0.42 $\pm$ 0.03
                       & 0.03034 $\pm$ 0.00006  & 29   & 0.25 $\pm$ 0.02\\
        NONS-ALD       & \textbf{0.00288} $\pm$ \textbf{0.00001}  & \textbf{58}   & 1.72 $\pm$ 0.12
                       & \textbf{0.02215} $\pm$ \textbf{0.00011}  & \textbf{29}   & 0.59 $\pm$ 0.04\\
        \Xhline{0.8pt}
        \multirow{2}{*}{Algorithm}&\multicolumn{3}{c|}{year , $\varsigma=16$}
        &\multicolumn{3}{c}{TomsHardware, $\varsigma=12$}\\
        \cline{2-7}&MSE &$J\vert D$       &Time (s) &MSE &$J\vert D$       &Time (s)       \\
        \hline
        FOGD           & 0.01501 $\pm$ 0.00004  & 400  & 4.04 $\pm$ 0.34
                       & 0.00080 $\pm$ 0.00003  & 400  & 2.28 $\pm$ 0.08\\
        NOGD           & 0.01511 $\pm$ 0.00013  & 400  & 16.49 $\pm$ 0.45
                       & 0.00085 $\pm$ 0.00001  & 400  & 10.52 $\pm$ 0.27\\
        PROS-N-KONS    & 0.01967 $\pm$ 0.00026  & 109  & 22.60 $\pm$ 3.33
                       & 0.00232 $\pm$ 0.00007  & 105  & 13.47 $\pm$ 1.45\\
         CON-KONS      & 0.01370 $\pm$ 0.00004  & 107  & 23.73 $\pm$ 3.02
                       & 0.00054 $\pm$ 0.00001  & 108  & 14.88 $\pm$ 1.72\\
        \hline
        AOGD-ALD       & 0.01499 $\pm$ 0.00002  & 106  & 3.72 $\pm$ 0.23
                       & 0.00062 $\pm$ 0.00000 & 100   &  1.97 $\pm$ 0.07\\
        NONS-ALD       & \textbf{0.01243} $\pm$ \textbf{0.00001} & \textbf{106}   & 31.65 $\pm$ 0.61
                       & \textbf{0.00043} $\pm$ \textbf{0.00000} & \textbf{100}   & 14.53 $\pm$ 0.38\\
        \Xhline{0.8pt}
      \end{tabular}
    \end{table}

    Table \ref{tab:ICML2023:experimental_results} shows the experimental results.
    We report the average mean squared error (MSE),
    the size of buffer ($J$), the number of random features ($D$), and the average per-round running time.
    The MSE is defined as $\mathrm{MSE}=\frac{1}{T}\sum^T_{t=1}(\hat{y}_t-y_t)^2$.

    As a whole,
    NONS-ALD enjoys the smallest MSE on all datasets.
    We first analyze the results of the three second-order algorithms, i.e.,
    NONS-ALD, CON-KONS and PROS-N-KONS.
    Both NONS-ALD and CON-KONS enjoy much better prediction performance than PROS-N-KONS.
    The MSE of PROS-N-KONS is even larger than that of the three first-order algorithms.
    The reason is that PROS-N-KONS uses the restart technique.
    If the times of restart are large, then the prediction performance will become bad.
    Both NONS-ALD and CON-KONS use carefully designed projection operations
    which keep the previous information.
    Besides,
    NONS-ALD performs better than CON-KONS
    which proves that our projection scheme in Lemma \ref{prop:ICML2023:approximate_A_t}
    and Lemma \ref{prop:ICML2023:approximate_w_t} is better than that of CON-KONS.
    All of the first-order algorithms have higher MSE than NONS-ALD and CON-KONS.
    The results are intuitive,
    since second-order algorithms use more information of the square loss function.
    The results very the first goal \textbf{G} 1.

    Next we analyze the size of buffer.
    Both AOGD-ALD and NONS-ALD only store few examples.
    For instance,
    AOGD-ALD and NONS-ALD only store 13 examples on the \textit{parkinson} dataset,
    and store 148 examples on the \textit{bank} dataset.
    NOGD stores $400$ examples, but still performs worse than our algorithms.
    CON-KONS and PROS-N-KONS also store more examples than our algorithms.
    It is worth mentioning that
    we must carefully tune the kernel function on each dataset.
    For instance,
    we set $\varsigma=8$ for the \textit{parkinson} dataset,
    while we set $\varsigma=16$ for the \textit{year} dataset.
    The results very the second goal \textbf{G} 2.

    Finally,
    we analyze the average per-round running time.
    As a whole,
    the running time of AOGD-ALD and NONS-ALD is comparable with all of the baseline algorithms.
    AOGD-ALD even runs fastest on all datasets except for the \textit{bank} dataset.
    The per-round time complexity of AOGD-ALD is $O(\min\{dJ+J^2,dT\})$.
    The average per-round time complexity of NONS-ALD is $O(dJ+J^2+\frac{J^4}{T})$.
    The smaller $J$ is, the faster AOGD-ALD and NONS-ALD will run.
    Note that changing the value of $D$ in FOGD and $J$ in NOGD
    will balance the prediction performance and computational cost.
    FOGD and NOGD can not perform better than NONS-ALD by increasing the value of $D$ or $J$.
    The reason is that FOGD and NOGD are first-order algorithm.
    The results very the third goal \textbf{G} 3.

\section{Reanalyze FOGD}

    In this section,
    we reanalyze the regret of FOGD \cite{Lu2016Large},
    and aim to prove a regret of $O(\frac{\sqrt{TL(f)\ln\frac{1}{\delta}}}{\sqrt{D}})$.
    Our proof is similar with the proof of Theorem 1 in \citet{Lu2016Large}.
    Thus we just show the critical differences.

    For any $f=\sum^T_{t=1}a_t\kappa({\bf x}_t,\cdot)\in\mathcal{H}$,
    we define $\tilde{f}=\sum^T_{t=1}a_t\tilde{\phi}({\bf x}_t)$,
    where $\tilde{\phi}(\cdot)$ is the explicit feature mapping constructed by
    the random feature technique \cite{Rahimi2007Random}.
    The regret can be decomposed as follows,
    \begin{align*}
        \mathrm{Reg}(f)
        =&\sum^T_{t=1}\ell({\bf w}^\top_t\tilde{\phi}_t({\bf x}_t),y_t)
        -\sum^T_{t=1}\ell(\tilde{f}({\bf x}_t),y_t)
        +\sum^T_{t=1}\ell(\tilde{f}({\bf x}_t),y_t)
        -\sum^T_{t=1}\ell(f({\bf x}_t),y_t)\\
        =&\underbrace{\sum^T_{t=1}\ell({\bf w}^\top_t\tilde{\phi}_t({\bf x}_t),y_t)
        -\sum^T_{t=1}\ell({\bf w}^\top\tilde{\phi}_t({\bf x}_t),y_t)}_{\mathcal{T}_1}
        +\underbrace{\sum^T_{t=1}\ell(\tilde{f}({\bf x}_t),y_t)
        -\sum^T_{t=1}\ell(f({\bf x}_t),y_t)}_{\mathcal{T}_2}.
    \end{align*}
    Following the original analysis of FOGD, $\mathcal{T}_1$ can be upper bounded as follows,
    $$
        \mathcal{T}_1\leq \frac{\Vert{\bf w}\Vert^2_2}{2\eta}
        +\frac{\eta}{2}\sum^T_{t=1}\ell({\bf w}^\top_t\tilde{\phi}_t({\bf x}_t),y_t)
        \leq \frac{\Vert{\bf w}\Vert^2_2+1}{2}\sqrt{\sum^T_{t=1}\ell({\bf w}^\top_t\tilde{\phi}_t({\bf x}_t),y_t)}
        \leq \frac{\Vert{\bf w}\Vert^2_2+1}{2}\sqrt{L(\tilde{f})}+\frac{(\Vert{\bf w}\Vert^2_2+1)^2}{4},
    $$
    where we define the learning rate
    $\eta=\frac{1}{\sqrt{\sum^T_{t=1}\ell({\bf w}^\top_t\tilde{\phi}_t({\bf x}_t),y_t)}}$.
    Next we analyze $\mathcal{T}_2$.
    The random feature technique guarantees that,
    with probability at least $1-2^8(\sigma_pR/\epsilon)^2\exp(-D\epsilon^2/4(d+2))$,
    $\vert \tilde{\phi}^\top({\bf x}_\tau)\tilde{\phi}({\bf x}_t)
    -\kappa({\bf x}_\tau,{\bf x}_t)\vert\leq \epsilon$.
    We further obtain
    \begin{align*}
        \mathcal{T}_2
        \leq&\sum^T_{t=1}\vert\ell'(\tilde{f}({\bf x}_t),y_t)\vert\cdot \Vert f\Vert_1\epsilon
        \leq 2\Vert f\Vert_1\epsilon\cdot \sqrt{T\sum^T_{t=1}\ell(\tilde{f}({\bf x}_t),y_t)}
        \leq  2\Vert f\Vert_1\epsilon\cdot \sqrt{T\sum^T_{t=1}\ell(f({\bf x}_t),y_t)}
        +4T\Vert f\Vert^2_1\epsilon^2,
    \end{align*}
    where $\Vert f\Vert_1=\sum^T_{t=1}\vert a_t\vert$.
    It was proved that $\Vert{\bf w}\Vert^2_2\leq(1+\epsilon)\Vert f\Vert^2_1$ \cite{Lu2016Large}.
    Combining the upper bounds on $\mathcal{T}_1$ and $\mathcal{T}_2$ gives that,
    with probability at least $1-\delta$,
    $$
        \mathrm{Reg}(f)
        =O\left(\Vert f\Vert^2_1\frac{T}{D}\ln\frac{1}{\delta}+(\Vert f\Vert^2_1+1)\sqrt{L(f)}
        +\frac{\Vert f\Vert_1\epsilon\cdot\sqrt{TL(f)}}{\sqrt{D}}\sqrt{\ln\frac{1}{\delta}}\right).
    $$
    We conclude the proof.
    In Table \ref{tab:ICML2023:comparison_results},
    we omit the term $O(\frac{T}{D})$.

\section{Proof of Theorem \ref{thm:ICML2023:AOGD-ALD}}

    \begin{proof}[Proof of Theorem \ref{thm:ICML2023:AOGD-ALD}]
        We first consider the case $\vert S_t\vert< \lfloor(\sqrt{d^2+4dT}-d)/2\rfloor$ for all $t=1,\ldots,T$.
        \begin{align*}
            \forall f\in\mathbb{H},\quad\mathrm{Reg}(f)\leq&\sum^T_{t=1}\langle \hat{\nabla}_t,f_t-f\rangle
            +\langle \nabla_t-\hat{\nabla}_t,f_t-f\rangle\\
            \leq&\sum^T_{t=1}\frac{1}{\eta_t}\langle f_t-\bar{f}_{t+1},f_t-f\rangle
            +\Vert f_t-f\Vert_{\mathcal{H}}
            \cdot\vert \ell'(f_t({\bf x}_t),y_t)\vert\cdot\sqrt{\alpha_t}\cdot\mathbb{I}\{\alpha_t\leq\alpha\}\\
            \leq&\sum^T_{t=1}\frac{1}{2\eta_t}
            \left[\Vert f_t-f\Vert^2_{\mathcal{H}}
            -\Vert\bar{f}_{t+1}-f\Vert^2_{\mathcal{H}}
            +\Vert \bar{f}_{t+1}-f_t\Vert^2_{\mathcal{H}}\right]
            +2U\sqrt{\alpha}\sqrt{T\sum^T_{t=1}\vert \ell'(f_t({\bf x}_t),y_t)\vert^2}\\
            \leq&\frac{3U^2}{2\eta_T}
            +\frac{1}{2}\sum^T_{t=1}\eta_t\Vert \hat{\nabla}_t\Vert^2_{\mathcal{H}}
            +4UT^{\frac{1-\zeta}{2}}\sqrt{\sum^T_{t=1} \ell(f_t({\bf x}_t),y_t)},
        \end{align*}
        where we define $\alpha=T^{-\zeta}$, $\zeta\geq 0$.
        Recalling the definition of $\eta_t$.
        It is easy to prove that
        $$
            \sum^T_{t=1}\frac{\Vert\hat{\nabla}_t\Vert^2_{\mathcal{H}}}
            {\sqrt{1+\sum^{t}_{\tau=1}\Vert\hat{\nabla}_{\tau}\Vert^2_{\mathcal{H}}}}
            \leq 2\sqrt{\sum^{T}_{\tau=1}\Vert\hat{\nabla}_{\tau}\Vert^2_{\mathcal{H}}}.
        $$
        Let $\zeta= 1$.
        The final regret satisfies
        \begin{align*}
            \mathrm{Reg}(f)
            \leq\frac{3U}{2}\sqrt{1+\sum^{T}_{\tau=1}\Vert\hat{\nabla}_{\tau}\Vert^2_{\mathcal{H}}}
            +U\sqrt{\sum^{T}_{\tau=1}\Vert\hat{\nabla}_{\tau}\Vert^2_{\mathcal{H}}}
            +4U\sqrt{\sum^T_{t=1} \ell(f_t({\bf x}_t),y_t)}
            \leq9U\sqrt{\sum^T_{t=1} \ell(f_t({\bf x}_t),y_t)}+3U.
        \end{align*}
        Solving for $\sum^T_{t=1}\ell(f_t({\bf x}_t),y_t)$ gives
        $$
            \mathrm{Reg}(f)
            \leq 9U\sqrt{\sum^T_{t=1}\ell(f({\bf x}_t),y_t)+3U}+81U^2.
        $$
        Next we consider that there exists a $t_0<T$ such that
        $\vert S_{t_0-1}\vert\leq \lfloor(\sqrt{d^2+4dT}-d)/2\rfloor$
        and $\vert S_{t_0}\vert> \lfloor(\sqrt{d^2+4dT}-d)/2\rfloor$.
        For $t\geq t_0$,
        our algorithm just runs OGD
        which is equivalent $\mathrm{ALD}_t$ does not hold.
        \begin{align*}
            \mathrm{Reg}(f)
            =&\sum^{t_0-1}_{t=1}[\ell(f_t({\bf x}_t),y_t)-\ell(f({\bf x}_t),y_t)]
            +\sum^{T}_{t=t_0}[\ell(f_t({\bf x}_t),y_t)-\ell(f({\bf x}_t),y_t)]\\
            \leq&\sum^{t_0-1}_{t=1}\frac{1}{2\eta_t}
            \left[\Vert f_t-f\Vert^2_{\mathcal{H}}
            -\Vert\bar{f}_{t+1}-f\Vert^2_{\mathcal{H}}
            +\Vert \bar{f}_{t+1}-f_t\Vert^2_{\mathcal{H}}\right]
            +2UT^{-\frac{\zeta}{2}}\sqrt{(t_0-1)\sum^{t_0-1}_{t=1}\vert \ell'(f_t({\bf x}_t),y_t)\vert^2}+\\
            &\sum^{T}_{t=t_0}\frac{1}{2\eta_t}
            \left[\Vert f_t-f\Vert^2_{\mathcal{H}}
            -\Vert\bar{f}_{t+1}-f\Vert^2_{\mathcal{H}}
            +\Vert \bar{f}_{t+1}-f_t\Vert^2_{\mathcal{H}}\right]\\
        \leq&9U\sqrt{\sum^T_{t=1} \ell(f_t({\bf x}_t),y_t)}+3U.
        \end{align*}
        Solving for $\sum^T_{t=1}\ell(f_t({\bf x}_t),y_t)$ gives the desired result.
    \end{proof}

\section{Proof of Theorem \ref{thm:ICML2023:KONS_implicit_computing}}

    We first give a technical lemma which has been stated in \cite{Calandriello2017Second}.
    \begin{lemma}
    \label{lemma:ICML2023:matrix_inverse}
        For any ${\bf X}\in\mathbb{R}^{n\times m}$ and $\alpha>0$,
        \begin{align*}
            {\bf X}{\bf X}^\top({\bf X}{\bf X}^\top+\alpha{\bf I})^{-1}
            =&{\bf X}({\bf X}^\top{\bf X}+\alpha{\bf I})^{-1}{\bf X}^{\top},\\
            ({\bf X}{\bf X}^\top+\alpha{\bf I})^{-1}
            =&\frac{1}{\alpha}({\bf I}-{\bf X}({\bf X}^\top{\bf X}+\alpha{\bf I})^{-1}{\bf X}^\top).
        \end{align*}
    \end{lemma}
    \begin{proof}[Proof of Theorem \ref{thm:ICML2023:KONS_implicit_computing}]
        For simplicity,
        let $g_t=\ell'(f_t(\mathbf{x}_t),y_t)$ and
        $$
            \hat{\phi}({\bf x}_t)
            =\left\{
            \begin{array}{ll}
                \phi(\mathbf{x}_t)& \mathrm{if}~\mathrm{ALD}_t~\mathrm{does~not~hold},\\
                \Phi_{S_t}{\bm \beta}^\ast_t&\mathrm{otherwise}.\\
            \end{array}
            \right.\\
        $$
        Let $\hat{\nabla}_t=g_t\hat{\phi}({\bf x}_t)$ and
        $\hat{{\bm \Phi}}_{t}=(\sqrt{\eta_1}\hat{\nabla}_1,\ldots,\sqrt{\eta_t}\hat{\nabla}_t)$.
        We can rewrite
        \begin{equation}
            {\bf A}_t
            ={\bf A}_{t-1}+\eta_t\hat{\nabla}_{t}\hat{\nabla}^\top_{t}
            =\mu{\bf I}+\sum^t_{\tau=1}\eta_{\tau}\hat{\nabla}_{\tau}\hat{\nabla}^\top_{\tau}
            =\mu{\bf I}+\hat{{\bm \Phi}}_{t}\hat{{\bm \Phi}}^\top_{t}.
        \end{equation}
        Recalling that
        \begin{align*}
            f_{t+1}=f_{t}-{\bf A}^{-1}_{t}\hat{\nabla}_{t}
            =f_{t-1}-{\bf A}^{-1}_{t-1}\hat{\nabla}_{t-1}-{\bf A}^{-1}_{t}\hat{\nabla}_{t}
            =\ldots
            =f_{1}-\sum^{t}_{\tau=1}{\bf A}^{-1}_{\tau}\hat{\nabla}_{\tau},
        \end{align*}
        in which $f_1=0$. Using Lemma \ref{lemma:ICML2023:matrix_inverse} yields
        \begin{align*}
            f_{t+1}({\bf x}_{t+1})=&f^\top_{t+1}\phi({\bf x}_{t+1})\\
            =&-\sum^{t}_{\tau=1}\hat{\nabla}^\top_{\tau}{\bf A}^{-1}_{\tau}\phi(\mathbf{x}_{t+1})\\
            =&-\sum^{t}_{\tau=1}
            \hat{\nabla}^\top_{\tau}\left(\mu{\bf I}+\hat{{\bm \Phi}}_{\tau}
            \hat{{\bm \Phi}}^\top_{\tau}\right)^{-1}\phi(\mathbf{x}_{t+1})\\
            =&\frac{-1}{\mu}\sum^{t}_{\tau=1}
            \hat{\nabla}^\top_{\tau}\left({\bf I}-\hat{{\bm \Phi}}_{\tau}
            (\hat{{\bm \Phi}}^\top_{\tau}\hat{{\bm \Phi}}_{\tau}
            +\mu{\bf I})^{-1}\hat{{\bm \Phi}}^\top_{\tau}\right)\phi(\mathbf{x}_{t+1})\\
            =&\frac{-1}{\mu}\sum^{t}_{\tau=1}
            g_\tau\cdot\left(\hat{\phi}({\bf x}_\tau)^\top\phi({\bf x}_{t+1})
            -\hat{\phi}({\bf x}_\tau)^\top\hat{{\bm \Phi}}_{\tau}
            (\hat{{\bm \Phi}}^\top_{\tau}\hat{{\bm \Phi}}_{\tau}
            +\mu{\bf I})^{-1}\hat{{\bm \Phi}}^\top_{\tau}\phi(\mathbf{x}_{t+1})\right),
        \end{align*}
        which concludes the proof.
    \end{proof}

\section{Proof of Lemma \ref{prop:ICML2023:approximate_A_t}}

    \begin{proof}[Proof of Lemma \ref{prop:ICML2023:approximate_A_t}]
        Recalling that
        \begin{align*}
            {\bf A}_j(s_j-1)&-\mu{\bf I}
            =\sum^{j-1}_{r=1}\sum_{t\in T_r}\eta_tg^2_r(t)\tilde{\phi}_j({\bf x}_t)
            \tilde{\phi}^\top_j({\bf x}_t)\\
            =&\sum^{j-1}_{r=1}\sum_{t\in T_r}\eta_tg^2_r(t)\mathcal{P}^\frac{1}{2}_{S(j)}
            {\bm\Phi}_{S(r)}{\bm \beta}^\ast_r
            ({\bm\Phi}_{S(r)}{\bm \beta}^\ast_r)^\top(\mathcal{P}^\frac{1}{2}_{S(j)})^\top\\
            =&\sum^{j-1}_{r=1}\sum_{t\in T_r}\eta_tg^2_r(t)\mathcal{P}^\frac{1}{2}_{S(j)}
            {\bm\Phi}_{S(r)}({\bm\Phi}^\top_{S(r)}{\bm\Phi}_{S(r)})^{-1}{\bm\Phi}^\top_{S(r)}\phi({\bf x}_t)
            ({\bm\Phi}_{S(r)}({\bm\Phi}^\top_{S(r)}{\bm\Phi}_{S(r)})^{-1}{\bm\Phi}^\top_{S(r)}\phi({\bf x}_t))^\top(\mathcal{P}^\frac{1}{2}_{S(j)})^\top\\
            =&\sum^{j-1}_{r=1}\sum_{t\in T_r}\eta_tg^2_r(t)\mathcal{P}^\frac{1}{2}_{S(j)}
            \mathcal{P}_{S(r)}\phi({\bf x}_t)
            (\mathcal{P}_{S(r)}\phi({\bf x}_t))^\top(\mathcal{P}^\frac{1}{2}_{S(j)})^\top\\
            =&\sum^{j-1}_{r=1}\sum_{t\in T_r}\eta_tg^2_r(t)\mathcal{P}^\frac{1}{2}_{S(j)}
            \mathcal{P}_{S(j-1)}\mathcal{P}_{S(r)}\phi({\bf x}_t)
            (\mathcal{P}_{S(j-1)}\mathcal{P}_{S(r)}\phi({\bf x}_t))^\top(\mathcal{P}^\frac{1}{2}_{S(j)})^\top\\
            =&\mathcal{P}^\frac{1}{2}_{S(j)}
            \left((\mathcal{P}^\frac{1}{2}_{S(j-1)})^\top\sum^{j-2}_{r=1}\sum_{t\in T_r}\eta_tg^2_r(t)
            \mathcal{P}^\frac{1}{2}_{S(j-1)}\mathcal{P}_{S(r)}\phi({\bf x}_t)
            (\mathcal{P}_{S(r)}
            \phi({\bf x}_t))^\top
            (\mathcal{P}^{\frac{1}{2}}_{S(j-1)})^\top\mathcal{P}^{\frac{1}{2}}_{S(j-1)}+\right.\\
            &\left.
            (\mathcal{P}^{\frac{1}{2}}_{S(j-1)})^\top\sum_{t\in T_{j-1}}\eta_tg^2_{j-1}(t)
            \mathcal{P}^{\frac{1}{2}}_{S(j-1)}\phi({\bf x}_t)
            (\mathcal{P}^{\frac{1}{2}}_{S(j-1)}\phi({\bf x}_t))^\top
            \mathcal{P}^{\frac{1}{2}}_{S(j-1)}\right)(\mathcal{P}^\frac{1}{2}_{S(j)})^\top\\
            =&\mathcal{P}^\frac{1}{2}_{S(j)}(\mathcal{P}^\frac{1}{2}_{S(j-1)})^\top
            \left({\bf A}_{j-1}(s_{j-1}-1)-\mu{\bf I}+
            \sum_{t\in T_{j-1}}\eta_tg^2_{j-1}(t)\phi_{j-1}({\bf x}_t)\phi^\top_{j-1}({\bf x}_t)
            \right)\mathcal{P}^{\frac{1}{2}}_{S(j-1)}(\mathcal{P}^\frac{1}{2}_{S(j)})^\top\\
            =&\mathcal{P}^\frac{1}{2}_{S(j)}(\mathcal{P}^\frac{1}{2}_{S(j-1)})^\top
            ({\bf A}_{j-1}(s_{j}-1)-\mu{\bf I})
            \mathcal{P}^\frac{1}{2}_{S(j-1)}(\mathcal{P}^\frac{1}{2}_{S(j)})^\top.
        \end{align*}
        Thus we can obtain
        $$
            {\bf A}_j(s_j-1)=\mu{\bf I}+
            \sum^{j-1}_{r=1}\sum_{t\in T_r}\eta_tg^2_r(t)\tilde{\phi}_j({\bf x}_t)
            \tilde{\phi}^\top_j({\bf x}_t)
            =\mu{\bf I}+\mathcal{P}^\frac{1}{2}_{S(j)}(\mathcal{P}^\frac{1}{2}_{S(j-1)})^\top
            ({\bf A}_{j-1}(s_{j}-1)-\mu{\bf I})
            \mathcal{P}^\frac{1}{2}_{S(j-1)}(\mathcal{P}^\frac{1}{2}_{S(j)})^\top,
        $$
        which concludes the proof.
    \end{proof}

\section{Proof of Lemma \ref{prop:ICML2023:approximate_w_t}}

    \begin{proof}[Proof of Lemma \ref{prop:ICML2023:approximate_w_t}]
        We directly use the definition of ${\bf w}_{j}(s_{j})$.
        \begin{align*}
            &{\bf w}_{j}(s_{j})
            =\mathcal{P}^{\frac{1}{2}}_{S(j)}(\mathcal{P}^{\frac{1}{2}}_{S(j-1)})^\top{\bf w}_{j-1}(s_{j})\\
           \Rightarrow &
           (\mathcal{P}^{\frac{1}{2}}_{S(j)}(\mathcal{P}^{\frac{1}{2}}_{S(j-1)})^\top)^\top{\bf w}_{j}(s_{j})
            =(\mathcal{P}^{\frac{1}{2}}_{S(j)}(\mathcal{P}^{\frac{1}{2}}_{S(j-1)})^\top)^\top
            \mathcal{P}^{\frac{1}{2}}_{S(j)}(\mathcal{P}^{\frac{1}{2}}_{S(j-1)})^\top{\bf w}_{j-1}(s_{j})\\
           \Rightarrow &
           \mathcal{P}^{\frac{1}{2}}_{S(j-1)}(\mathcal{P}^{\frac{1}{2}}_{S(j)})^\top{\bf w}_{j}(s_{j})
            =\mathcal{P}^{\frac{1}{2}}_{S(j-1)}(\mathcal{P}^{\frac{1}{2}}_{S(j)})^\top
            \mathcal{P}^{\frac{1}{2}}_{S(j)}(\mathcal{P}^{\frac{1}{2}}_{S(j-1)})^\top{\bf w}_{j-1}(s_{j})\\
           \Rightarrow &
           \mathcal{P}^{\frac{1}{2}}_{S(j-1)}(\mathcal{P}^{\frac{1}{2}}_{S(j)})^\top{\bf w}_{j}(s_{j})
            =\mathcal{P}^{\frac{1}{2}}_{S(j-1)}\mathcal{P}_{S(j)}
            (\mathcal{P}^{\frac{1}{2}}_{S(j-1)})^\top{\bf w}_{j-1}(s_{j})\\
           \Rightarrow &
           \mathcal{P}^{\frac{1}{2}}_{S(j-1)}(\mathcal{P}^{\frac{1}{2}}_{S(j)})^\top{\bf w}_{j}(s_{j})
            ={\bf w}_{j-1}(s_{j}),
        \end{align*}
        where we use the fact
        $\mathcal{P}^{\frac{1}{2}}_{S(j-1)}(\mathcal{P}^{\frac{1}{2}}_{S(j-1)})^\top={\bf I}$.
        Next we prove ${\bf w}_{j}(s_j)\in\mathbb{W}_{s_j}$.
        \begin{align*}
            ({\bf w}_{j}(s_{j}))^\top\phi_j({\bf x}_{s_j})
            =({\bf w}_{j-1}(s_{j}))^\top
            \mathcal{P}^{\frac{1}{2}}_{S(j-1)}(\mathcal{P}^{\frac{1}{2}}_{S(j)})^\top
            \mathcal{P}^{\frac{1}{2}}_{S(j)}\phi({\bf x}_{s_j})
            =&({\bf w}_{j-1}(s_{j}))^\top
            \mathcal{P}^{\frac{1}{2}}_{S(j-1)}\phi({\bf x}_{s_j})\\
            =&({\bf w}_{j-1}(s_{j}))^\top\phi_{j-1}({\bf x}_{s_j}).
        \end{align*}
        Since $\vert ({\bf w}_{j-1}(s_{j}))^\top\phi_{j-1}({\bf x}_{s_j})\vert\leq U$,
        it must be ${\bf w}_{j}(s_j)\in\mathbb{W}_{s_j}$.
        Thus we conclude the proof.
    \end{proof}

\section{Proof of Lemma \ref{lemma:ICML2023:spectral_norm_error_kernel_approximate}}

    We first prove a technique lemma.
    \begin{lemma}
    \label{lemma:ICML2023:pointwise_kernel_approximate_error}
        For all $j=1,\ldots,J$,
        let $\mathcal{P}_{S(j)}$ be the projection matrix onto the column space of ${\bm \Phi}_{S(j)}$.
        For all $t\in T_j$,
        $$
            0\leq \kappa(\mathbf{x}_t,\mathbf{x}_t)-
            \phi^\top_j(\mathbf{x}_t)\phi_j(\mathbf{x}_t)\leq \alpha,
        $$
        For any $r\in[J]$ and $t\in T_r$,
        denote by
        $\tilde{\phi}_J({\bf x}_t)=\mathcal{P}^{\frac{1}{2}}_{S(J)}{\bm \Phi}_{S(r)}{\bm \beta}^\ast_r(t)$.
        Then for any $i,j\in[J]$ and for any $t\in T_i, \tau\in T_j$,
        $$
            \kappa(\mathbf{x}_t,\mathbf{x}_\tau)-
            \tilde{\phi}^\top_J(\mathbf{x}_t)\tilde{\phi}_J(\mathbf{x}_\tau)\leq \sqrt{\alpha}.
        $$
    \end{lemma}
    \begin{proof}[Proof of Lemma \ref{lemma:ICML2023:pointwise_kernel_approximate_error}]
        For any $t\in T_j$,
        $S_t=S(j)$.
        \begin{align*}
            \kappa(\mathbf{x}_t,\mathbf{x}_t)-
            \phi^\top_j(\mathbf{x}_t)\phi_j(\mathbf{x}_t)
            &=\kappa(\mathbf{x}_t,\mathbf{x}_t)-\phi(\mathbf{x}_t)\mathcal{P}_{S(j)}\phi(\mathbf{x}_t)\\
            &= \kappa(\mathbf{x}_t,\mathbf{x}_t)-
            ({\bf \Phi}^\top_{S(j)}\kappa(\mathbf{x}_{t},\cdot))^\top{\bf K}^{-1}_{S(j)}
            {\bf \Phi}^\top_{S(j)}\kappa(\mathbf{x}_{t},\cdot)\\
            &=\alpha_t\in[0,\alpha],
        \end{align*}
        where we use the fact that the $\mathrm{ALD}_t$ condition holds.\\
        Next we consider $t\in T_i$ and $\tau\in T_j$.
        Without loss of generality,
        assuming that $i<j$.
        \begin{align*}
            \kappa({\bf x}_t,{\bf x}_\tau)-\tilde{\phi}^\top_J(\mathbf{x}_t)\tilde{\phi}_J(\mathbf{x}_\tau)
            =&\phi({\bf x}_t)^\top\phi({\bf x}_{\tau})-
            \left(\mathcal{P}^{\frac{1}{2}}_{S(J)}{\bm \Phi}_{S(i)}{\bm \beta}^\ast_i(t)\right)^\top
            \mathcal{P}^{\frac{1}{2}}_{S(J)}{\bm \Phi}_{S(j)}{\bm \beta}^\ast_j(\tau)\\
            =&\phi({\bf x}_t)^\top\phi({\bf x}_{\tau})-
            \left(\mathcal{P}_{S(i)}\phi({\bf x}_t)\right)^\top
            \mathcal{P}_{S(J)}\mathcal{P}_{S(j)}\phi({\bf x}_{\tau})\\
            =&\phi({\bf x}_t)^\top\phi({\bf x}_{\tau})-
            \phi({\bf x}_t)^\top\mathcal{P}_{S(i)}
            \mathcal{P}_{S(J)}\mathcal{P}_{S(j)}\phi({\bf x}_{\tau})\\
            =&\phi({\bf x}_t)^\top\phi({\bf x}_{\tau})-
            \phi({\bf x}_t)^\top\mathcal{P}_{S(i)}\phi({\bf x}_{\tau})\\
            \leq&
            \sqrt{\Vert \phi({\bf x}_t)-
            \mathcal{P}_{S(i)}\phi({\bf x}_t)\Vert^2_{\mathcal{H}}}\cdot
            \Vert\phi({\bf x}_\tau)\Vert_{\mathcal{H}}\\
            =&\sqrt{\phi({\bf x}_t)^\top\phi({\bf x}_t)-
            \phi({\bf x}_t)^\top\mathcal{P}_{S(i)}\phi({\bf x}_t)}\cdot
            \Vert\phi({\bf x}_\tau)\Vert_{\mathcal{H}}\\
            \leq&\sqrt{\alpha},
        \end{align*}
        which concludes the proof.
    \end{proof}

    \begin{proof}[Proof of Lemma \ref{lemma:ICML2023:spectral_norm_error_kernel_approximate}]
        Denote by ${\bf \Phi}_{T_j}=(\phi({\bf x}_{s_j}),\phi({\bf x}_{s_j+1}),\ldots,\phi({\bf x}_{s_{j+1}-1}))
        \in\mathbb{R}^{n\times \vert T_j\vert}$,
        ${\bf K}_{T_j}={\bf \Phi}^\top_{T_j}{\bf \Phi}_{T_j}$
        and $\tilde{{\bf K}}_{T_j}={\bf \Phi}^\top_{T_j}
        \mathcal{P}_{S(j)}{\bf \Phi}_{T_j}$.
        Let
        ${\bf K}_{-}={\bf K}_{T_j}-\tilde{{\bf K}}_{T_j}$.
        We first prove that ${\bf K}_{-}$ is a positive semi-definite (PSD) matrix.
        Let ${\bf \Phi}_{[e_j]}=(\phi({\bf x}_{s_1}),\phi({\bf x}_{s_2}),
        \ldots,\phi({\bf x}_{s_j}),\phi({\bf x}_{s_j+1}),\ldots,\phi({\bf x}_{s_{j+1}-1}))
        \in\mathbb{R}^{n\times (\vert T_j\vert+j-1)}$,
        ${\bf K}^{-1}_{[e_j]}={\bf \Phi}^\top_{[e_j]}{\bf \Phi}_{[e_j]}$, and
        $\mathcal{P}_{[e_j]}={\bf \Phi}_{[e_j]}{\bf K}^{-1}_{[e_j]}{\bf \Phi}^\top_{[e_j]}$
        be the projection matrix on the column space of ${\bm \Phi}_{[e_j]}$.
        We have
        \begin{align*}
            {\bf K}_{T_j}-\tilde{{\bf K}}_{T_j}
            =&{\bm \Phi}^\top_{T_j}\mathcal{P}_{[e_j]}{\bm \Phi}_{T_j}
            -{\bm \Phi}^\top_{T_j}\mathcal{P}_{S(j)}{\bm \Phi}_{T_j}\\
            =&{\bm \Phi}^\top_{I_j}(\mathcal{P}^\top_{[e_j]}\mathcal{P}_{[e_j]}
            -\mathcal{P}^\top_{S(j)}\mathcal{P}_{S(j)})
            {\bm \Phi}_{I_j}\\
            =&{\bm \Phi}^\top_{I_j}(\mathcal{P}_{[e_j]}-\mathcal{P}_{S(j)})^\top
            (\mathcal{P}_{[e_j]}-\mathcal{P}_{S(j)}){\bm \Phi}_{T_j}\\
            =&\left((\mathcal{P}_{[e_j]}-\mathcal{P}_{S(j)}){\bm \Phi}_{T_j}\right)^\top
            (\mathcal{P}_{[e_j]}-\mathcal{P}_{S(j)}){\bm \Phi}_{T_j},
        \end{align*}
        where $\mathcal{P}_{[e_j]}$ and $\mathcal{P}_{S(j)}$ satisfy
        $\mathcal{P}^\top_{[e_j]}\mathcal{P}_{[e_j]}=\mathcal{P}_{[e_j]}$ and
        $\mathcal{P}^\top_{S(j)}\mathcal{P}_{S(j)}=\mathcal{P}_{S(j)}$.
        Besides,
        $$
            \mathcal{P}^\top_{[e_j]}\mathcal{P}_{S(j)}
            =\mathcal{P}_{[e_j]}{\bm \Phi}_{S(j)}({\bm \Phi}^\top_{S(j)}{\bm \Phi}_{S(j)})^{-1}
            {\bm \Phi}_{S(j)}^\top
            =\mathcal{P}_{S(j)},
        $$
        where ${\bm \Phi}_{S(j)}$ belongs to the column space of ${\bm \Phi}_{[e_j]}$.
        For any ${\bf a}\in\mathbb{R}^{\vert T_j\vert}$,
        we have
        \begin{align*}
            {\bf a}^\top{\bf K}_{-}{\bf a}
            =\left((\mathcal{P}_{[e_j]}-\mathcal{P}_{S(j)}){\bm \Phi}_{T_j}{\bf a}\right)^\top
            (\mathcal{P}_{[e_j]}-\mathcal{P}_{S(j)}){\bm \Phi}_{I_j}{\bf a}
            =\left\Vert (\mathcal{P}_{[e_j]}-\mathcal{P}_{S(j)})
            {\bm \Phi}_{T_j}{\bf a}\right\Vert^2_{\mathcal{H}}
            \geq 0.
        \end{align*}
        Thus ${\bf K}_{-}$ is a PSD matrix.
        Lemma \ref{lemma:ICML2023:pointwise_kernel_approximate_error}
        gives ${\bf K}_{-}[i,i]\in[0,\alpha]$.
        Thus we have
        $$
            \left\Vert {\bf K}_{T_j}-\tilde{{\bf K}}_{T_j}\right\Vert_2
            \leq \mathrm{tr}\left({\bf K}_{T_j}-\tilde{{\bf K}}_{T_j}\right)
            \leq \vert T_j\vert\cdot\alpha.
        $$
        Let $\tilde{{\bm \Phi}}_T=\left((\tilde{\phi}_J({\bf x}_t))_{t\in T_1},\ldots,
        (\tilde{\phi}_J({\bf x}_t))_{t\in T_J}\right)$.
        The second statement in Lemma \ref{lemma:ICML2023:pointwise_kernel_approximate_error}
        can derive
        $$
            \left\Vert {\bf K}_T-\tilde{{\bm \Phi}}^\top_T\tilde{{\bm \Phi}}_T\right\Vert_2
            \leq \left\Vert {\bf K}_T-\tilde{{\bm \Phi}}^\top_T\tilde{{\bm \Phi}}_T\right\Vert_F
            \leq T\sqrt{\alpha},
        $$
        which concludes the proof.
    \end{proof}

\section{Proof of Theorem \ref{thm:ICML2023:NONS-ALD:regret_bound}}

    We first give some technical lemmas.

\subsection{Technical Lemmas}

    \begin{lemma}
    \label{lemma:ICML2023:continuous_projection_1}
        For any $f\in\mathbb{H}$,
        let $f_j$ be the projection of $f$ onto the column space of ${\bm \Phi}_{S(j)}$,
        and $f_{j+1}$ be the projection of $f$ onto the column space of ${\bm \Phi}_{S(j+1)}$.
        The following three claims hold:
        (i) There exist ${\bf w}_j=\mathcal{P}^\frac{1}{2}_{S(j)}f\in\mathbb{R}^j$
        and
        ${\bf w}_{j+1}=\mathcal{P}^\frac{1}{2}_{S(j+1)}f\in\mathbb{R}^{j+1}$,
        such that $f_j({\bf x})={\bf w}^\top_j\phi_j({\bf x})$ and
        $f_{j+1}({\bf x})={\bf w}^\top_{j+1}\phi_{j+1}({\bf x})$,
        (ii) ${\bf w}_j\in\cap^T_{t=1}\mathbb{W}_t$ and ${\bf w}_{j+1}\in\cap^T_{t=1}\mathbb{W}_t$,
        (iii)
        $
            {\bf w}_j=\mathcal{P}^{\frac{1}{2}}_{S(j)}(\mathcal{P}^{\frac{1}{2}}_{S(j+1)})^\top
            {\bf w}_{j+1}.
        $
    \end{lemma}
    \begin{proof}[Proof of Lemma \ref{lemma:ICML2023:continuous_projection_1}]
        For any $f\in\mathbb{H}$,
        the projection of $f$ on the column space of ${\bm \Phi}_{S(j+1)}$ and ${\bm \Phi}_{S(j)}$ are
        $$
            f_{j+1} =\mathcal{P}_{S(j+1)}f,\quad f_j=\mathcal{P}_{S(j)}f=\mathcal{P}_{S(j)}\mathcal{P}_{S(j+1)}f
            =\mathcal{P}_{S(j)}f_{j+1}.
        $$
        We have
        $$
            f_{j}({\bf x}_t)=(\mathcal{P}_{S(j)}f)^\top\phi({\bf x}_t)
            =f^\top(\mathcal{P}^\frac{1}{2}_{S(j)})^\top\mathcal{P}^\frac{1}{2}_{S(j)}\phi({\bf x}_t)
            =(\mathcal{P}^\frac{1}{2}_{S(j)}f)^\top\phi_j({\bf x}_t)
            ={\bf w}^\top_j\phi_j({\bf x}_t).
        $$
        Thus we obtain
        $$
            {\bf w}_j=\mathcal{P}^{\frac{1}{2}}_{S(j)}f,\quad
            {\bf w}_{j+1}=\mathcal{P}^{\frac{1}{2}}_{S(j+1)}f,
        $$
        which concludes the first claim.\\
        For the second claim,
        we have
        $$
            \forall t\in[T],~\vert {\bf w}^\top_j\phi_j({\bf x}_t)\vert
            =\vert (\mathcal{P}^{\frac{1}{2}}_{S(j)}f)^\top\mathcal{P}^{\frac{1}{2}}_{S(j)}\phi({\bf x}_t)\vert
            =\vert f^\top\mathcal{P}_{S(j)}\phi({\bf x}_t)\vert
            \leq \Vert f\Vert_{\mathcal{H}}\leq U.
        $$
        Thus ${\bf w}_j\in \cap^T_{t=1}\mathbb{W}_t$. Similarly, we have ${\bf w}_{j+1}\in\cap^T_{t=1}\mathbb{W}_t$.

        Since $\mathcal{P}^{\frac{1}{2}}_{S(j)}=\mathcal{P}^{\frac{1}{2}}_{S(j)}\mathcal{P}_{S(j+1)}$,
        we have,
        $$
            {\bf w}_j=\mathcal{P}^{\frac{1}{2}}_{S(j)}\mathcal{P}_{S(j+1)}f
            =\mathcal{P}^{\frac{1}{2}}_{S(j)}(\mathcal{P}^{\frac{1}{2}}_{S(j+1)})^\top\mathcal{P}^{\frac{1}{2}}_{S(j+1)}f
            =\mathcal{P}^{\frac{1}{2}}_{S(j)}(\mathcal{P}^{\frac{1}{2}}_{S(j+1)})^\top{\bf w}_{j+1}
        $$
        which concludes the third claim.
    \end{proof}

    \begin{lemma}[\cite{Hazan2007Logarithmic}]
    \label{lemma:ICML2023:matrix_determinant}
        Let ${\bf u}_t\in\mathbb{R}^j$ for $t=1,\ldots,T$ be a sequence of vectors
        such that for some $r>0$,
        $\Vert{\bf u}_t\Vert\leq r$.
        Let $\mu>0$.
        Define ${\bf V}_t=\sum^t_{\tau=1}{\bf u}_\tau{\bf u}^\top_\tau+\mu{\bf I}$.
        Then
        $$
            \sum^T_{t=1}{\bf u}^\top_t{\bf V}^{-1}_t{\bf u}_t
            \leq\sum^T_{t=1}\ln\frac{\mathrm{det}({\bf V}_t)}{\mathrm{det}({\bf V}_{t-1})}
            =\ln\frac{\mathrm{det}({\bf V}_T)}{\mathrm{det}({\bf V}_0)}
            =\ln\mathrm{det}\left(\frac{1}{\mu}\sum^T_{\tau=1}{\bf u}_\tau{\bf u}^\top_\tau+{\bf I}\right),
        $$
        where ${\bf V}_0=\mu{\bf I}$.
    \end{lemma}
    \begin{lemma}
    \label{lemma:ICML2023:continuous_covariance_matrix_projection}
        For any $j=1,\ldots,J$,
        $$
           \frac{\mathrm{Det}({\bf A}_j(s_{j}-1))}{\mathrm{Det}({\bf A}_{j-1}(s_{j}-1))}= \mu.
        $$
    \end{lemma}
    \begin{proof}[Proof of Lemma \ref{lemma:ICML2023:continuous_covariance_matrix_projection}]
        For any $r\leq j$ and $t\in T_r$,
        let $\bar{\phi}({\bf x}_t)=\sqrt{\eta_t}g_r(t)\phi({\bf x}_t)$.
        Recalling that
        \begin{align*}
            {\bf A}_{j-1}(s_{j}-1)
            =&{\bf A}_{j-1}(s_{j-1}-1)
            +\sum_{t\in T_{j-1}}\eta_tg^2_{j-1}(t)\phi_{j-1}({\bf x}_t)\phi^\top_{j-1}({\bf x}_t)\\
            =&\mu{\bf I}+\sum^{j-2}_{r=1}\sum_{t\in T_r}\eta_tg^2_r(t)\tilde{\phi}_{j-1}({\bf x}_t)
            \tilde{\phi}^\top_{j-1}({\bf x}_t)
            +\sum_{t\in T_{j-1}}\eta_tg^2_{j-1}(t)\mathcal{P}^{\frac{1}{2}}_{S(j-1)}\phi({\bf x}_t)
            (\mathcal{P}^{\frac{1}{2}}_{S(j-1)}\phi({\bf x}_t))^\top\\
            =&\mu{\bf I}+\sum^{j-1}_{r=1}\sum_{t\in T_r}
            \mathcal{P}^{\frac{1}{2}}_{S(j-1)}\mathcal{P}_{S(r)}\bar{\phi}({\bf x}_t)
            (\mathcal{P}^{\frac{1}{2}}_{S(j-1)}\mathcal{P}_{S(r)}\bar{\phi}({\bf x}_t))^\top\\
            =&\mu{\bf I}+\bar{{\bm \Phi}}_{S(j-1)}\bar{{\bm \Phi}}^\top_{S(j-1)},
        \end{align*}
        where
        \begin{align*}
            \bar{{\bm \Phi}}_{S(j-1)}
            =&\mathcal{P}^{\frac{1}{2}}_{S(j-1)}
            \left[\left(\mathcal{P}_{S(r)}\bar{\phi}({\bf x}_t)\right)_{t\in T_r}\right]_{r\in[j-1]}
            \in\mathbb{R}^{(j-1)\times \sum^{j-1}_{r=1}\vert T_r\vert}.
        \end{align*}
        Similarly,
        we have
        \begin{align*}
            {\bf A}_j(s_{j}-1)
            =&\mu{\bf I}+\sum^{j-1}_{r=1}\sum_{t\in T_r}\eta_tg^2_r(t)\tilde{\phi}_{j}({\bf x}_t)
            \tilde{\phi}^\top_{j}({\bf x}_t)\\
            =&\mu{\bf I}+\sum^{j-1}_{r=1}\sum_{t\in T_r}
            \mathcal{P}^{\frac{1}{2}}_{S(j)}\mathcal{P}_{S(r)}\bar{\phi}({\bf x}_t)
            (\mathcal{P}^{\frac{1}{2}}_{S(j)}\mathcal{P}_{S(r)}\bar{\phi}({\bf x}_t))^\top\\
            =&\mu{\bf I}+\bar{{\bm \Phi}}_{S(j)}\bar{{\bm \Phi}}^\top_{S(j)},
        \end{align*}
        where we define
        $$
            \bar{{\bm \Phi}}_{S(j)}
            =\mathcal{P}^{\frac{1}{2}}_{S(j)}
            \left[\left(\mathcal{P}_{S(r)}\bar{\phi}({\bf x}_t)\right)_{t\in T_r}\right]_{r\in[j-1]}
            \in\mathbb{R}^{j\times \sum^{j-1}_{r=1}\vert T_r\vert}.
        $$
        We have the following two facts.
        \begin{align*}
            \mathrm{rank}(\bar{{\bm \Phi}}_{S(j-1)}\bar{{\bm \Phi}}^\top_{S(j-1)})=
            \mathrm{rank}(\bar{{\bm \Phi}}^\top_{S(j-1)}\bar{{\bm \Phi}}_{S(j-1)}),\quad
            \mathrm{rank}(\bar{{\bm \Phi}}_{S(j)}\bar{{\bm \Phi}}^\top_{S(j)})=
            \mathrm{rank}(\bar{{\bm \Phi}}^\top_{S(j)}\bar{{\bm \Phi}}_{S(j)}).
        \end{align*}
        We can prove
        \begin{align*}
            \bar{{\bm \Phi}}^\top_{S(j)}\bar{{\bm \Phi}}_{S(j)}
            -\bar{{\bm \Phi}}^\top_{S(j-1)}\bar{{\bm \Phi}}_{S(j-1)}
            =&
            \left(\left[\left(\mathcal{P}_{S(r)}\bar{\phi}({\bf x}_t)\right)_{t\in T_r}\right]_{r\in[j-1]}\right)^\top
            \mathcal{P}_{S(j)}
            \left[\left(\mathcal{P}_{S(r)}\bar{\phi}({\bf x}_t)\right)_{t\in T_r}\right]_{r\in[j-1]}-\\
            &\left(\left[\left(\mathcal{P}_{S(r)}\bar{\phi}({\bf x}_t)\right)_{t\in T_r}\right]_{r\in[j-1]}\right)^\top
            \mathcal{P}_{S(j-1)}
            \left[\left(\mathcal{P}_{S(r)}\bar{\phi}({\bf x}_t)\right)_{t\in T_r}\right]_{r\in[j-1]}\\
            =&\left(\left[\left(\mathcal{P}_{S(r)}\bar{\phi}({\bf x}_t)\right)_{t\in T_r}\right]_{r\in[j-1]}\right)^\top
            \left[\left(\mathcal{P}_{S(r)}\bar{\phi}({\bf x}_t)\right)_{t\in T_r}\right]_{r\in[j-1]}-\\
            &\left(\left[\left(\mathcal{P}_{S(r)}\bar{\phi}({\bf x}_t)\right)_{t\in T_r}\right]_{r\in[j-1]}\right)^\top
            \left[\left(\mathcal{P}_{S(r)}\bar{\phi}({\bf x}_t)\right)_{t\in T_r}\right]_{r\in[j-1]}\\
            =&0,
        \end{align*}
        in which we use the following facts
        \begin{align*}
            &\mathcal{P}_{S(j)}=(\mathcal{P}^{\frac{1}{2}}_{S(j)})^\top\mathcal{P}^{\frac{1}{2}}_{S(j)},\\
            &\mathcal{P}_{S(r)}\mathcal{P}_{S(j)}
            ={\bm \Phi}_{S(r)}({\bm \Phi}^\top_{S(r)}{\bm \Phi}_{S(r)}){\bm \Phi}^\top_{S(r)}\mathcal{P}_{S(j)}
            =\mathcal{P}_{S(r)}, r=1,\ldots,j-1,\\
            &\mathcal{P}_{S(r)}\mathcal{P}_{S(j-1)}
            ={\bm \Phi}_{S(r)}({\bm \Phi}^\top_{S(r)}{\bm \Phi}_{S(r)}){\bm \Phi}^\top_{S(r)}\mathcal{P}_{S(j-1)}
            =\mathcal{P}_{S(r)}, r=1,\ldots,j-1.
        \end{align*}
        It must be that $\bar{{\bm \Phi}}_{S(j)}\bar{{\bm \Phi}}^\top_{S(j)}$
        and $\bar{{\bm \Phi}}_{S(j-1)}\bar{{\bm \Phi}}^\top_{S(j-1)}$
        have the same non-zero eigenvalues,
        denoted by $\bar{\lambda}_1,\bar{\lambda}_2,\ldots,\bar{\lambda}_k$, $k\leq j-1$.
        We have
        $$
            \frac{\mathrm{Det}({\bf A}_j(s_{j}-1))}{\mathrm{Det}({\bf A}_{j-1}(s_{j}-1))}
            =\frac{\prod^j_{r=1}(\mu+\bar{\lambda}_r)}{\prod^{j-1}_{r=1}(\mu+\bar{\lambda}_r)}
            =\frac{\prod^k_{r=1}(\mu+\bar{\lambda}_r)\cdot\mu^{j-k}}{\prod^k_{r=1}(\mu+\bar{\lambda}_r)\cdot\mu^{j-k-1}}
            =\mu,
        $$
        which concludes the proof.
    \end{proof}

    \begin{proof}[Proof of Theorem \ref{thm:ICML2023:NONS-ALD:regret_bound}]
    Let $f_j$ be the projection of $f\in\mathbb{H}$ onto the column space of ${\bm \Phi}_{S(j)}$,
    $j=1,\ldots,J$.
    We decompose the regret into two components.
    \begin{align*}
    \forall f\in\mathbb{H},~\mathrm{Reg}(f)
        =&\sum^{J}_{j=1}\sum_{t\in T_j}[\ell(\hat{y}_t,y_t)-\ell(f({\bf x}_t),y_t)]\\
        =&\sum^{J}_{j=1}\sum_{t\in T_j}[\ell(\hat{y}_t,y_t)-\ell(f_j({\bf x}_t),y_t)]
        +\sum^{J}_{j=1}\sum_{t\in T_j}[\ell(f_j({\bf x}_t),y_t)-\ell(f({\bf x}_t),y_t)]\\
        =&\underbrace{\sum^{J}_{j=1}
        \sum_{t\in T_j}[\ell(\hat{y}_t,y_t)-\ell({\bf w}^\top_j\phi_j({\bf x}_t),y_t)]}_{\mathcal{T}_1}
        +\underbrace{\sum^{J}_{j=1}\sum_{t\in T_j}[\ell(f_j({\bf x}_t),y_t)-\ell(f({\bf x}_t),y_t)]}_{\mathcal{T}_2}.
    \end{align*}
    Lemma \ref{lemma:ICML2023:continuous_projection_1} proved
    that there is a ${\bf w}_j\in\mathbb{R}^j$ such that $f_j({\bf x}_t)={\bf w}^\top_j\phi_j({\bf x}_t)$.

\subsection{Analyze $\mathcal{T}_1$}

    We consider a fixed epoch $T_j$.
    At any round $t\in T_j$,
    the instantaneous regret can be upper bounded as follows
    \begin{align*}
        &\ell(\hat{y}_t,y_t)-\ell({\bf w}^\top_j\phi_j({\bf x}_t),y_t)\\
        =&(\hat{y}_t-y_t)^2-({\bf w}^\top_j\phi_j({\bf x}_t)-y_t)^2\\
        =&2(\hat{y}_t-y_t)(\hat{y}_t-{\bf w}^\top_j\phi_j({\bf x}_t))
        -(\hat{y}_t-{\bf w}^\top_j\phi_j({\bf x}_t))^2\\
        =&\langle \nabla \ell({\bf w}^\top_j(t)\phi_j({\bf x}_t)),{\bf w}_j(t)-{\bf w}_j\rangle
        -\frac{1}{4(\hat{y}_t-y_t)^2}
        \left(\langle \nabla \ell({\bf w}^\top_j(t)\phi_j({\bf x}_t)),
        {\bf w}_j(t)-{\bf w}_j\rangle\right)^2\\
        \leq&\langle \nabla \ell({\bf w}^\top_j(t)\phi_j({\bf x}_t)),{\bf w}_j(t)-{\bf w}_j\rangle
        -\frac{1}{8(U^2+Y^2)}
        \left(\langle \nabla \ell({\bf w}^\top_j(t)\phi_j({\bf x}_t)),
        {\bf w}_j(t)-{\bf w}_j\rangle\right)^2.
    \end{align*}
    For simplicity,
    denote by $\sigma=\frac{1}{8(U^2+Y^2)}$ and
    $\nabla_j(t)=\nabla \ell({\bf w}^\top_j(t)\phi_j({\bf x}_t))=\ell'(\hat{y}_t,y_t)\phi_j({\bf x}_t)$. \\
    Lemma \ref{lemma:ICML2023:continuous_projection_1}
    has proved that ${\bf w}_j\in\mathbb{W}_{t+1}$.
    Using the property of projection,
    we have
    \begin{align*}
        &\Vert {\bf w}_j(t+1)-{\bf w}_j\Vert^2_{{\bf A}_j(t)}-\Vert {\bf w}_j(t)-{\bf w}_j\Vert^2_{{\bf A}_j(t)}\\
        \leq&\Vert \tilde{{\bf w}}_j(t+1)-{\bf w}_j\Vert^2_{{\bf A}_j(t)}
        -\Vert {\bf w}_j(t)-{\bf w}_j\Vert^2_{{\bf A}_j(t)}\\
        =&\Vert {\bf w}_j(t)-{\bf A}^{-1}_j(t)\nabla_j(t)-{\bf w}_j\Vert^2_{{\bf A}_j(t)}
        -\Vert {\bf w}_j(t)-{\bf w}_j\Vert^2_{{\bf A}_j(t)}\\
        =&-2\langle {\bf w}_j(t)-{\bf w}_j,{\bf A}^{-1}_j(t)\nabla_j(t)\rangle_{{\bf A}_j(t)}
        +\Vert {\bf A}^{-1}_j(t)\nabla_j(t)\Vert^2_{{\bf A}_j(t)}\\
        =&-2\langle {\bf w}_j(t)-{\bf w}_j,\nabla_j(t)\rangle
        +\nabla^\top_j(t){\bf A}^{-1}_j(t)\nabla_j(t).
    \end{align*}
    Let $\eta_t=2\sigma$.
    Rearranging terms and summing over $t\in T_j=\{s_j,s_j+1,\ldots,s_{j+1}-1\}$ gives
    \begin{align*}
       &\sum^{s_{j+1}-1}_{t=s_j}\left(\langle {\bf w}_j(t)-{\bf w}_j,\nabla_j(t)\rangle
       - \sigma\left(\langle \nabla_j(t),{\bf w}_j(t)-{\bf w}_j\rangle\right)^2\right)\\
        \leq&\sum^{s_{j+1}-1}_{t=s_j}\left(\frac{\Vert {\bf w}_j(t)-{\bf w}_j\Vert^2_{{\bf A}_j(t)}
        -\Vert {\bf w}_j(t+1)-{\bf w}_j\Vert^2_{{\bf A}_j(t)}}{2}
        +\frac{\nabla^\top_j(t){\bf A}^{-1}_j(t)\nabla_j(t)}{2}
        -\sigma\Vert{\bf w}_j(t)-{\bf w}_j\Vert^2_{\nabla_j(t)\nabla^\top_j(t)}\right)\\
        =&\frac{\Vert {\bf w}_j(s_j)-{\bf w}_j\Vert^2_{{\bf A}_j(s_j)}}{2}-
        \frac{\Vert {\bf w}_j(s_{j+1})-{\bf w}_j\Vert^2_{{\bf A}_j(s_{j+1}-1)}}{2}+\\
        &\sum^{s_{j+1}-2}_{t=s_j}\frac{\Vert {\bf w}_j(t+1)-{\bf w}_j\Vert^2_{{\bf A}_j(t+1)}
        -\Vert {\bf w}_j(j+1)-{\bf w}_j\Vert^2_{{\bf A}_j(t)}}{2}
        +\sum^{s_{j+1}-1}_{t=s_j}\frac{\nabla^\top_j(t){\bf A}^{-1}_j(t)\nabla_j(t)}{2}-\\
        &\sum^{s_{j+1}-2}_{t=s_j}
        \sigma\Vert{\bf w}_j(t+1)-{\bf w}_j\Vert^2_{\nabla_j(j+1)\nabla^\top_j(j+1)}-
        \sigma\Vert{\bf w}_j(s_j)-{\bf w}_j\Vert^2_{\nabla_j(s_j)\nabla^\top_j(s_j)}\\
        =&\sum^{s_{t+1}-1}_{t=s_j}\frac{\nabla^\top_j(t){\bf A}^{-1}_j(t)\nabla_j(t)}{2}
        +\frac{\Vert {\bf w}_j(s_j)-{\bf w}_j\Vert^2_{{\bf A}_j(s_j-1)}}{2}-
        \frac{\Vert {\bf w}_j(s_{j+1})-{\bf w}_j\Vert^2_{{\bf A}_j(s_{j+1}-1)}}{2},
    \end{align*}
    where we use the following two facts
    \begin{align*}
        {\bf A}_j(t+1)=&{\bf A}_j(t)+2\sigma\nabla_j(t+1)\nabla^\top_j(t+1),\\
        {\bf A}_j(s_j-1)=&{\bf A}_j(s_j)-2\sigma\nabla_j(s_j)\nabla^\top_j(s_j).
    \end{align*}
    Summing over $j=1,2,\ldots,J$,
    we obtain
    \begin{align*}
        \mathcal{T}_1\leq \underbrace{\sum^J_{j=1}\sum^{s_{j+1}-1}_{t=s_j}
        \frac{\nabla^\top_j(t){\bf A}^{-1}_j(t)\nabla_j(t)}{2}}_{\mathcal{T}_{1,1}}
        +\underbrace{\sum^J_{j=1}\frac{\Vert {\bf w}_j(s_j)-{\bf w}_j\Vert^2_{{\bf A}_j(s_j-1)}}{2}-
        \frac{\Vert {\bf w}_j(s_{j+1})-{\bf w}_j\Vert^2_{{\bf A}_j(s_{j+1}-1)}}{2}}_{\mathcal{T}_{1,2}}.
    \end{align*}

    The key of our analysis is to prove tighter upper bounds on
    $\mathcal{T}_{1,2}$ and $\mathcal{T}_{1,1}$ using our initial configurations
    in Lemma \ref{prop:ICML2023:approximate_A_t}
    and Lemma \ref{prop:ICML2023:approximate_w_t}.
    We first give some high-level explanations on why our analysis can give tighter regret bound.

    The analysis of PROS-N-KONS \cite{Calandriello2017Efficient}
    initializes ${\bf w}_j(s_j)={\bf 0}\in\mathbb{R}^j$
    and ${\bf A}_j(s_j-1)=\alpha{\bf I}\in\mathbb{R}^{j\times j}$.
    A trivial upper bound on $\mathcal{T}_1$ can be derived, i.e.,
    \begin{align*}
        \mathcal{T}_1\leq J\cdot\max_{j=1,\ldots,J}\sum^{s_{j+1}-1}_{t=s_j}
        \frac{\nabla^\top_j(t){\bf A}^{-1}_j(t)\nabla_j(t)}{2}
        +\sum^J_{j=1}\frac{\Vert{\bf w}_j\Vert^2_{{\bf A}_j(s_j-1)}}{2}.
    \end{align*}
    It is naturally that the regret bound is linear with $J$.
    The analysis can not be improved unless we reset the initial configurations
    ${\bf w}_j(s_j)$ and ${\bf A}_j(s_j-1)$.
    Intuitively,
    there is a negative term
    $-\Vert {\bf w}_j(s_{j+1})-{\bf w}_j\Vert^2_{{\bf A}_j(s_{j+1}-1)}$ in $\mathcal{T}_{1,2}$.
    Our analysis will use this negative term to cancel with the next positive term
    $\Vert {\bf w}_{j+1}(s_{j+1})-{\bf w}_{j+1}\Vert^2_{{\bf A}_{j+1}(s_{j+1}-1)}$.
    To this end,
    we must carefully design ${\bf w}_{j+1}(s_{j+1})$.
    Finally,
    we will prove $\mathcal{T}_{1,2}=\frac{1}{2}\Vert f\Vert^2_{\mathcal{H}}$
    which is independent of $J$.
    Similar idea is used to analyze $\mathcal{T}_{1,1}$.

    We first analyze $\mathcal{T}_{1,2}$ and then analyze $\mathcal{T}_{1,1}$.

\subsubsection{Analyzing $\mathcal{T}_{1,2}$}

    Rearranging terms yields
    \begin{align*}
        \mathcal{T}_{1,2}=&\frac{\Vert {\bf w}_1(s_1)-{\bf w}_1\Vert^2_{{\bf A}_1(s_1-1)}}{2}
        -\frac{\Vert {\bf w}_J(s_{J+1})-{\bf w}_J\Vert^2_{{\bf A}_J(s_{J+1}-1)}}{2}+\\
        &\frac{1}{2}\sum^J_{j=1}\left[\Vert {\bf w}_{j+1}(s_{j+1})-{\bf w}_{j+1}\Vert^2_{{\bf A}_{j+1}(s_{j+1}-1)}-
        \Vert {\bf w}_j(s_{j+1})-{\bf w}_j\Vert^2_{{\bf A}_j(s_{j+1}-1)}\right].
    \end{align*}
    To upper bound the second term,
    the key is to analyze the relation between ${\bf A}_{j+1}(s_{j+1}-1)$
    and ${\bf A}_{j}(s_{j+1}-1)$.
    For any $r\leq j$ and $t\in T_r$,
    let
    $\bar{{\bm \Phi}}_{T_r}=(\bar{\phi}({\bf x}_t))_{t\in T_r}$
        where $\bar{\phi}({\bf x}_t)=\sqrt{\eta_t}g_r(t)\phi({\bf x}_t)=\sqrt{2\sigma}g_r(t)\phi({\bf x}_t)$.
    According to \eqref{eq:ICML2023:approximate_A_t},
    we have
    \begin{align*}
        {\bf A}_{j+1}(s_{j+1}-1)
        =&\mu{\bf I}+\sum^{j}_{r=1}\sum_{t\in T_r}\eta_tg^2_r(t)
        \mathcal{P}^{\frac{1}{2}}_{S(j+1)}{\bm \Phi}_{S(r)}{\bf \beta}^\ast_r(t)
        ({\bm \Phi}_{S(r)}{\bf \beta}^\ast_r(t))^\top(\mathcal{P}^{\frac{1}{2}}_{S(j+1)})^\top\\
        =&\mu{\bf I}+\sum^{j}_{r=1}\sum_{t\in T_r}\eta_tg^2_r(t)
        \mathcal{P}^{\frac{1}{2}}_{S(j+1)}\mathcal{P}_{S(r)}\phi({\bf x}_t)
        (\mathcal{P}_{S(r)}\phi({\bf x}_t))^\top(\mathcal{P}^{\frac{1}{2}}_{S(j+1)})^\top\\
        =&\mathcal{P}^\frac{1}{2}_{S(j+1)}\left[\mu{\bf I}+\left(
        \mathcal{P}_{S(r)}\bar{{\bm \Phi}}_{T_r}\right)_{r\in[j]}
        \left(\mathcal{P}_{S(r)}\bar{{\bm \Phi}}_{T_r}\right)^\top_{r\in[j]}\right]
        (\mathcal{P}^\frac{1}{2}_{S(j+1)})^\top,\\
        {\bf A}_{j}(s_{j+1}-1)
        =&\mu{\bf I}+\sum^{j}_{r=1}\sum_{t\in T_r}\eta_tg^2_r(t)
        \mathcal{P}^{\frac{1}{2}}_{S(j)}{\bm \Phi}_{S(r)}{\bf \beta}^\ast_r(t)
        ({\bm \Phi}_{S(r)}{\bf \beta}^\ast_r(t))^\top(\mathcal{P}^{\frac{1}{2}}_{S(j)})^\top\\
        =&\mathcal{P}^\frac{1}{2}_{S(j)}\left[\mu{\bf I}+\left(\mathcal{P}_{S(r)}
        \bar{{\bm \Phi}}_{T_r}\right)_{r\in[j]}
        \left(\mathcal{P}_{S(r)}\bar{{\bm \Phi}}_{T_r}\right)^\top_{r\in[j]}\right]
        (\mathcal{P}^\frac{1}{2}_{S(j)})^\top.
    \end{align*}
    For simplicity,
    let $\Delta={\bf w}_{j+1}(s_{j+1})-{\bf w}_{j+1}$.
    According to Lemma \ref{lemma:ICML2023:continuous_projection_1} and
    Lemma \ref{prop:ICML2023:approximate_w_t},
    we obtain
    \begin{align*}
        &\Vert {\bf w}_j(s_{j+1})-{\bf w}_j\Vert^2_{{\bf A}_j(s_{j+1}-1)}\\
        =&\Vert \mathcal{P}^{\frac{1}{2}}_{S(j)}(\mathcal{P}^{\frac{1}{2}}_{S(j+1)})^\top
        ({\bf w}_{j+1}(s_{j+1})-{\bf w}_{j+1})\Vert^2_{{\bf A}_j(s_{j+1}-1)}\\
        =&\Delta^\top\mathcal{P}^{\frac{1}{2}}_{S(j+1)}(\mathcal{P}^{\frac{1}{2}}_{S(j)})^\top
        {\bf A}_j(s_{j+1}-1)\mathcal{P}^{\frac{1}{2}}_{S(j)}(\mathcal{P}^{\frac{1}{2}}_{S(j+1)})^\top\Delta\\
        =&\Delta^\top\mathcal{P}^{\frac{1}{2}}_{S(j+1)}(\mathcal{P}^{\frac{1}{2}}_{S(j)})^\top
        \mathcal{P}^\frac{1}{2}_{S(j)}\left[\mu{\bf I}+\left(\mathcal{P}_{S(r)}
        \bar{{\bm \Phi}}_{T_r}\right)_{r\in[j]}
        \left(\mathcal{P}_{S(r)}
        \bar{{\bm \Phi}}_{T_r}\right)^\top_{r\in[j]}\right](\mathcal{P}^\frac{1}{2}_{S(j)})^\top
        \mathcal{P}^{\frac{1}{2}}_{S(j)}(\mathcal{P}^{\frac{1}{2}}_{S(j+1)})^\top\Delta\\
        =&\Delta^\top\mathcal{P}^{\frac{1}{2}}_{S(j+1)}
        \mathcal{P}_{S(j)}\left[\mu{\bf I}+\left(\mathcal{P}_{S(r)}
        \bar{{\bm \Phi}}_{T_r}\right)_{r\in[j]}
        \left(\mathcal{P}_{S(r)}\bar{{\bm \Phi}}_{T_r}\right)^\top_{r\in[j]}\right]
        \mathcal{P}_{S(j)}(\mathcal{P}^{\frac{1}{2}}_{S(j+1)})^\top\Delta\\
        =&\Delta^\top\mathcal{P}^{\frac{1}{2}}_{S(j+1)}
        \left[\mu{\bf I}+\left(\mathcal{P}_{S(r)}
        \bar{{\bm \Phi}}_{T_r}\right)_{r\in[j]}
        \left(\mathcal{P}_{S(r)}\bar{{\bm \Phi}}_{T_r}\right)^\top_{r\in[j]}
        +\mu\mathcal{P}_{S(j)}-\mu{\bf I}\right]
        (\mathcal{P}^{\frac{1}{2}}_{S(j+1)})^\top\Delta\\
        =&\Vert {\bf w}_{j+1}(s_{j+1})-{\bf w}_{j+1}\Vert^2_{{\bf A}_{j+1}(s_{j+1}-1)}
        +\mu\Delta^\top\mathcal{P}^{\frac{1}{2}}_{S(j+1)}
        \left(\mathcal{P}_{S(j)}-{\bf I}\right)
        (\mathcal{P}^{\frac{1}{2}}_{S(j+1)})^\top\Delta\\
        =&\Vert {\bf w}_{j+1}(s_{j+1})-{\bf w}_{j+1}\Vert^2_{{\bf A}_{j+1}(s_{j+1}-1)}
        -\mu\Delta^\top\mathcal{P}^{\frac{1}{2}}_{S(j+1)}
        \left({\bf I}-\mathcal{P}_{S(j)}\right)^\top\left({\bf I}-\mathcal{P}_{S(j)}\right)
        (\mathcal{P}^{\frac{1}{2}}_{S(j+1)})^\top\Delta\\
        =&\Vert {\bf w}_{j+1}(s_{j+1})-{\bf w}_{j+1}\Vert^2_{{\bf A}_{j+1}(s_{j+1}-1)}
        -\mu\left\Vert\left({\bf I}-\mathcal{P}_{S(j)}\right)
        (\mathcal{P}^{\frac{1}{2}}_{S(j+1)})^\top\Delta\right\Vert^2\\
        =&\Vert {\bf w}_{j+1}(s_{j+1})-{\bf w}_{j+1}\Vert^2_{{\bf A}_{j+1}(s_{j+1}-1)}
        -\mu\left\Vert\left({\bf I}-\mathcal{P}_{S(j)}\right)
        (\mathcal{P}^{\frac{1}{2}}_{S(j+1)})^\top(\mathcal{P}^{\frac{1}{2}}_{S(j+1)}(\mathcal{P}^{\frac{1}{2}}_{S(j)})^\top{\bf w}_{j}(s_{j+1})
        -\mathcal{P}^{\frac{1}{2}}_{S(j+1)}f)\right\Vert^2\\
        =&\Vert {\bf w}_{j+1}(s_{j+1})-{\bf w}_{j+1}\Vert^2_{{\bf A}_{j+1}(s_{j+1}-1)}
        -\mu\left\Vert\left({\bf I}-\mathcal{P}_{S(j)}\right)
        (\mathcal{P}_{S(j+1)}(\mathcal{P}^{\frac{1}{2}}_{S(j)})^\top{\bf w}_{j}(s_{j+1})
        -\mathcal{P}_{S(j+1)}f)\right\Vert^2\\
        =&\Vert {\bf w}_{j+1}(s_{j+1})-{\bf w}_{j+1}\Vert^2_{{\bf A}_{j+1}(s_{j+1}-1)}
        -\mu\left\Vert\left(\mathcal{P}_{S(j)}-{\bf I}\right)\mathcal{P}_{S(j+1)}f\right\Vert^2,
    \end{align*}
    where the last but one equality satisfies
    $$
        \left({\bf I}-\mathcal{P}_{S(j)}\right)
        \mathcal{P}_{S(j+1)}(\mathcal{P}^{\frac{1}{2}}_{S(j)})^\top
        =\mathcal{P}_{S(j+1)}(\mathcal{P}^{\frac{1}{2}}_{S(j)})^\top
        -\mathcal{P}_{S(j)}(\mathcal{P}^{\frac{1}{2}}_{S(j)})^\top
        =(\mathcal{P}^{\frac{1}{2}}_{S(j)})^\top
        -\mathcal{P}_{S(j)}(\mathcal{P}^{\frac{1}{2}}_{S(j)})^\top
        =0.
    $$
    Thus we can obtain
    \begin{align*}
        \mathcal{T}_{1,2}
        \leq& \frac{1}{2}\left(\Vert {\bf w}_1(s_1)-{\bf w}_1\Vert^2_{{\bf A}_1(s_1-1)}
        +\sum^J_{j=1}\mu\left\Vert\left(\mathcal{P}_{S(j)}-{\bf I}\right)
        \mathcal{P}_{S(j+1)}f\right\Vert^2\right)\\
        =& \frac{1}{2}\left(\Vert {\bf w}_1\Vert^2_{\mu{\bf I}}
        +\sum^J_{j=1}\mu f^\top(\mathcal{P}_{S(j+1)}-\mathcal{P}_{S(j)})f\right)\\
        \leq& \frac{1}{2}\left(\Vert {\bf w}_1\Vert^2_{\mu{\bf I}}+
        \mu f^\top(\mathcal{P}_{S(J+1)}-\mathcal{P}_{S(1)})f\right)\\
        \leq& \frac{\mu}{2}\left(\Vert {\bf w}_1\Vert^2_2+ \Vert f\Vert^2_{\mathcal{H}}
        - f^\top\mathcal{P}_{S(1)}f\right)\\
        \leq&\frac{\mu}{2}\Vert f\Vert^2_{\mathcal{H}},
    \end{align*}
    where ${\bf w}_1(s_1)={\bf 0}$, ${\bf A}_1(s_1-1)=\mu{\bf I}$ and ${\bf w}_1=\mathcal{P}_{S(1)}f$.\\

\subsubsection{analyzing $\mathcal{T}_{1,1}$}

    Recalling that
    \begin{align*}
        \mathcal{T}_{1,1}=
        \sum^{J}_{j=1}\sum^{s_{j+1}-1}_{t=s_j}\frac{\nabla^\top_j(t){\bf A}^{-1}_j(t)\nabla_j(t)}{2},
    \end{align*}
    where $\nabla_j(t)=\ell'(\hat{y}_t,y_t)\phi_j({\bf x}_t), t\in T_j$.
    According to \eqref{eq:ICML2023:tilde_phi},
    we have
    $$
        \forall t\in T_j,\quad\phi_j({\bf x}_t)=\mathcal{P}^{\frac{1}{2}}_{S(j)}\phi({\bf x}_t)=
        \mathcal{P}^{\frac{1}{2}}_{S(j)}\mathcal{P}_{S(j)}\phi({\bf x}_t)=
        \mathcal{P}^{\frac{1}{2}}_{S(j)}{\bm \Phi}_{S(j)}{\bm \beta}^\ast_j(t)
        =\tilde{\phi}_j({\bf x}_t).
    $$
    For any $r\leq j, t\in T_r$,
    denote by $\tilde{\nabla}_j(t)=g_r(t)\tilde{\phi}_j({\bf x}_t)$.
    We can rewrite ${\bf A}_j(t)$ as follows
    $$
        {\bf A}_j(t)
        ={\bf A}_{j}(s_j-1)+\sum^t_{\tau=s_j}\eta_{\tau}
        g^2_j(\tau)\phi_j({\bf x}_\tau)\phi^\top_j({\bf x}_\tau)
        =\mu{\bf I}+2\sigma\sum^{j-1}_{r=1}\sum_{\tau\in T_r}
        \tilde{\nabla}_j(\tau)\tilde{\nabla}^\top_j(\tau)
        +2\sigma\sum^{t}_{\tau=s_j}\tilde{\nabla}_j(\tau)\tilde{\nabla}^\top_j(\tau).
    $$
    Using Lemma \ref{lemma:ICML2023:matrix_determinant},
    we obtain
    \begin{align*}
        \sum^{s_{j+1}-1}_{t=s_j}\nabla^\top_j(t){\bf A}^{-1}_j(t)\nabla_j(t)
        =\sum^{s_{j+1}-1}_{t=s_j}\tilde{\nabla}^\top_j(t){\bf A}^{-1}_j(t)\tilde{\nabla}_j(t)
        \leq\frac{1}{2\sigma}\sum^{s_{j+1}-1}_{t=s_j}\ln\frac{\mathrm{Det}({\bf A}_j(t))}{\mathrm{Det}({\bf A}_j(t-1))}
        =\frac{\ln\frac{\mathrm{Det}({\bf A}_j(s_{j+1}-1))}{\mathrm{Det}({\bf A}_j(s_j-1))}}{2\sigma}.
    \end{align*}
    Summing over $j=1\ldots,J$ yields
    \begin{align}
        \sum^J_{j=1}\sum^{s_{j+1}-1}_{t=s_j}\nabla^\top_j(t){\bf A}^{-1}_j(t)\nabla_j(t)
        =&\frac{1}{2\sigma}
        \sum^J_{j=1}\ln\frac{\mathrm{Det}({\bf A}_j(s_{j+1}-1))}{\mathrm{Det}({\bf A}_j(s_j-1))}\nonumber\\
        =&\frac{1}{2\sigma}
        \ln\prod^J_{j=1}\frac{\mathrm{Det}({\bf A}_j(s_{j+1}-1))}{\mathrm{Det}({\bf A}_j(s_j-1))}\nonumber\\
        =&\frac{1}{2\sigma}
        \ln\frac{1}{\mathrm{Det}({\bf A}_1(s_1-1))}\cdot
        \prod^J_{j=2}\frac{\mathrm{Det}({\bf A}_{j-1}(s_{j}-1))}{\mathrm{Det}({\bf A}_j(s_{j}-1))}
        \cdot \mathrm{Det}({\bf A}_J(s_{J+1}-1))\nonumber\\
        \overbrace{=}^{(\ast)}&\frac{1}{2\sigma}\ln\frac{\mathrm{Det}({\bf A}_J(s_{J+1}-1))}
        {\mu^{J-1}\mathrm{Det}({\bf A}_1(s_1-1))}\nonumber\\
        \overbrace{=}^{(\ast\ast)}&\frac{1}{2\sigma}\ln\mathrm{Det}\left(\frac{1}{\mu}\sum^J_{j=1}
        \sum_{t\in T_j}2\sigma\tilde{\nabla}_J(t)\tilde{\nabla}^\top_J(t)+{\bf I}\right)
        \label{eq:ICML2023:regret_continuous_updating_A_t}
    \end{align}
    where $(\ast)$ comes from Lemma \ref{lemma:ICML2023:continuous_covariance_matrix_projection},
    and $(\ast\ast)$ comes from ${\bf A}_1(s_1-1)=\mu$.\\
    For simplicity, let
    $$
        \tilde{{\bm \Phi}}=\sqrt{2\sigma}\left[(\tilde{\nabla}_J(t))_{t\in T_j}\right]_{j\in[J]}
        \in\mathbb{R}^{J\times T},\quad
        \bar{\bm \Phi}=\sqrt{2\sigma}\left[(g_j(t)\phi({\bf x}_t))_{t\in T_j}\right]_{j\in[J]}
        \in\mathbb{R}^{n\times T}.
    $$
    Using the second statement of Lemma \ref{lemma:ICML2023:pointwise_kernel_approximate_error},
    we can obtain
    $$
        \left\Vert \tilde{{\bm \Phi}}^\top\tilde{{\bm \Phi}}
        -\bar{\bm \Phi}^\top\bar{\bm \Phi}\right\Vert_2
        \leq \left\Vert \tilde{{\bm \Phi}}^\top\tilde{{\bm \Phi}}
        -\bar{\bm \Phi}^\top\bar{\bm \Phi}\right\Vert_F
        \leq \sqrt{2\sigma}\cdot
        \sqrt{T^2\cdot\max_{i,j\in[J]}\vert g_i(t)\vert\cdot\vert g_j(t)\vert\alpha}\leq T\sqrt{2\alpha}.
    $$
    Thus $\tilde{{\bm \Phi}}^\top\tilde{{\bm \Phi}}\preceq
    \bar{{\bm \Phi}}^\top\bar{{\bm \Phi}}+T\sqrt{2\alpha}{\bf I}$.
    We further obtain
    \begin{align*}
        \ln\mathrm{Det}\left(\frac{2\sigma}{\mu}\sum^J_{r=1}
        \sum_{t\in T_r}\tilde{\nabla}_J(t)\tilde{\nabla}^\top_J(t)+{\bf I}\right)
        =\ln\mathrm{Det}\left(\frac{\tilde{{\bm \Phi}}^\top\tilde{{\bm \Phi}}}{\mu}+{\bf I}\right)
        \leq\ln\mathrm{Det}\left(\frac{\bar{{\bm \Phi}}^\top\bar{{\bm \Phi}}}{\mu}
        +\left(\frac{T\sqrt{2\alpha}}{\mu}+1\right){\bf I}\right).
    \end{align*}
    Let $\bar{\lambda}_1\geq \bar{\lambda}_2\geq \ldots\geq \bar{\lambda}_T$ be the eigenvalues of
    $\bar{{\bm \Phi}}^\top\bar{{\bm \Phi}}$. Then we have
    \begin{align*}
        \ln\mathrm{Det}\left(\frac{\bar{{\bm \Phi}}^\top\bar{{\bm \Phi}}}{\mu}
        +\left(\frac{T\sqrt{2\alpha}}{\mu}+1\right){\bf I}\right)
        =&\ln\left(\prod^T_{i=1}
        \left(\frac{\bar{\lambda}_i}{\mu}+\frac{T\sqrt{2\alpha}}{\mu}+1\right)\right)\\
        \leq&\ln\left((1+\frac{T\sqrt{2\alpha}}{\mu})^T\prod^T_{i=1}
        \left(\frac{\bar{\lambda}_i}{\mu}+1\right)\right)\\
        =&T\ln\left(1+\frac{T\sqrt{2\alpha}}{\mu}\right)+
        \sum^T_{i=1}\ln \left(\frac{\bar{\lambda}_i}{\mu}+1\right).
    \end{align*}
    Let ${\bf D}=(\{g_j(t)\sqrt{2\sigma}\}_{t\in \cup^J_{j=1}T_j})$.
    Then $\bar{{\bm \Phi}}^\top\bar{{\bm \Phi}}={\bf D}{\bf K}_T{\bf D}$.
    Since $\ln(1+x)<\frac{x}{1+x}(1+\ln(1+x))$ for all $x>0$,
    we have
    \begin{align*}
        \ln\mathrm{Det}\left(\frac{\bar{{\bm \Phi}}^\top\bar{{\bm \Phi}}}{\mu}
        +\left(\frac{T\sqrt{2\alpha}}{\mu}+1\right){\bf I}\right)
        \leq&T\ln\left(1+\frac{T\sqrt{2\alpha}}{\mu}\right)+
        \sum^T_{i=1}\frac{\bar{\lambda}_i}{\mu+\bar{\lambda}_i}
        \left(1+\max_i\ln \frac{\bar{\lambda}_i+\mu}{\mu}\right)\\
        \leq&T\ln\left(1+\frac{T\sqrt{2\alpha}}{\mu}\right)
        +\mathrm{tr}(\bar{{\bm \Phi}}^\top\bar{{\bm \Phi}}(\bar{{\bm \Phi}}^\top\bar{{\bm \Phi}}+\mu{\bf I})^{-1})
        \cdot
        \left(1+\ln \frac{\mathrm{tr}(\bar{{\bm \Phi}}^\top\bar{{\bm \Phi}})+\mu}{\mu}\right)\\
        \overbrace{\leq}^{(\ast)}&T\ln\left(1+\frac{T\sqrt{2\alpha}}{\mu}\right)
        +\mathrm{tr}\left({\bf K}_T({\bf K}_T+\frac{\mu}{2}{\bf I})^{-1}\right)\cdot
        \left(1+\ln \frac{2T+\mu}{\mu}\right).
    \end{align*}
    $(\ast)$ follows the proof of Theorem 1 in \citep{Calandriello2017Second}
    which states
    $$
       \mathrm{tr}
       \left(\bar{{\bm \Phi}}^\top\bar{{\bm \Phi}}(\bar{{\bm \Phi}}^\top\bar{{\bm \Phi}}+\mu{\bf I})^{-1}\right)
       =\mathrm{tr}\left({\bf K}_T({\bf K}_T+\mu{\bf D}^{-2})^{-1}\right)
       \leq \mathrm{tr}\left({\bf K}_T({\bf K}_T+\mu\lambda_{\min}({\bf D}^{-2}){\bf I})^{-1}\right)
       = \mathrm{d}_{\mathrm{eff}}\left(\frac{\mu}{2}\right),
    $$
    where $\lambda_{\min}({\bf D}^{-2})=\frac{1}{2\sigma\max(g_r(t))^2}
    =\frac{4(U^2+Y^2)}{\max(g_r(t))^2}=\frac{1}{2}$.
    We obtain
    $$
        \mathcal{T}_{1,1} \leq \frac{1}{2}T\ln\left(1+\frac{T\sqrt{2\alpha}}{\mu}\right)+
        \frac{1}{2}\mathrm{d}_{\mathrm{eff}}\left(\frac{\mu}{2}\right)\cdot\left(1+\ln \frac{2T+\mu}{\mu}\right)
        \leq\frac{T^2\sqrt{\alpha}}{\sqrt{2}\mu}
        +\frac{1}{2}\mathrm{d}_{\mathrm{eff}}\left(\frac{\mu}{2}\right)\cdot\left(1+\ln \frac{2T+\mu}{\mu}\right),
    $$
    where we use the fact $\ln(1+x)\leq x$ for all $x\geq 0$.

\subsection{Analyze $\mathcal{T}_2$}

        Let ${\bf Y}_{T_j}=(y_{s_j},y_{s_j+1},\ldots,y_{s_{j+1}-1})^\top$.
        Recalling that $f_j=\mathcal{P}_{S(j)}f$.
        We have
        \begin{align*}
            \mathcal{T}_2
            =&\sum^{J}_{j=1}\sum_{t\in T_j}\left((\mathcal{P}_{S(j)}f)^\top\phi({\bf x}_t)-y_t\right)^2
            -\sum^T_{t=1}\left(f^\top\phi({\bf x}_t)-y_t\right)^2\\
            =&\sum^{J}_{j=1}\Vert f^\top\mathcal{P}_{S(j)}{\bm \Phi}_{T_j}-{\bf Y}_{T_j}\Vert^2_2
            -\sum^{J}_{j=1}\Vert f^\top{\bm \Phi}_{T_j}-{\bf Y}_{T_j}\Vert^2_2\\
            =&\sum^{J}_{j=1}\Vert f^\top{\bm \Phi}_{T_j}
            -{\bf Y}_{T_j}+f^\top(\mathcal{P}_{S(j)}{\bm \Phi}_{T_j}-{\bm \Phi}_{T_j})\Vert^2_2
            -\sum^{J}_{j=1}\Vert f^\top{\bm \Phi}_{T_j}-{\bf Y}_{T_j}\Vert^2_2\\
            =&\sum^{J}_{j=1}f^\top(\mathcal{P}_{S(j)}{\bm \Phi}_{T_j}-{\bm \Phi}_{T_j})
            (\mathcal{P}_{S(j)}{\bm \Phi}_{T_j}-{\bm \Phi}_{T_j})^\top f
            +2\sum^{J}_{j=1}\langle f^\top{\bm \Phi}_{T_j}-{\bf Y}_{T_j},
            f^\top(\mathcal{P}_{S(j)}{\bm \Phi}_{T_j}-{\bm \Phi}_{T_j})\rangle\\
            =&\sum^{J}_{j=1}\Vert f\Vert^2_{\mathcal{H}}\cdot
            \Vert\mathcal{P}_{S(j)}{\bm \Phi}_{T_j}-{\bm \Phi}_{T_j}\Vert^2_2
            +2\sum^{J}_{j=1}\Vert f^\top{\bm \Phi}_{T_j}-{\bf Y}_{T_j}\Vert_2\cdot
            \Vert f\Vert_{\mathcal{H}}\cdot \Vert \mathcal{P}_{S(j)}{\bm \Phi}_{T_j}-{\bm \Phi}_{T_j}\Vert_2.
        \end{align*}
        Using the first statement of Lemma \ref{lemma:ICML2023:spectral_norm_error_kernel_approximate},
        we obtain
        \begin{align*}
            \Vert\mathcal{P}_{S(j)}{\bm \Phi}_{T_j}-{\bm \Phi}_{T_j}\Vert^2_2
            =&\Vert(\mathcal{P}_{S(j)}{\bm \Phi}_{T_j}-{\bm \Phi}_{T_j})^\top(\mathcal{P}_{S(j)}{\bm \Phi}_{T_j}-{\bm \Phi}_{T_j})\Vert_2\\
            =&\Vert{\bm \Phi}^\top_{T_j}{\bm \Phi}_{T_j}
            -{\bm \Phi}^\top_{T_j}\mathcal{P}_{S(j)}{\bm \Phi}_{T_j}\Vert_2\\
            \leq& \vert T_j\vert\alpha.
        \end{align*}
        Thus we have
        $$
            \mathcal{T}_2\leq \Vert f\Vert^2_{\mathcal{H}}\cdot T\alpha
            +2\sqrt{\sum^{J}_{j=1}\Vert f^\top{\bm \Phi}_{T_j}-{\bf Y}_{T_j}\Vert^2_2}\cdot
            \Vert f\Vert_{\mathcal{H}}\cdot \sqrt{\sum^{J}_{j=1}\vert T_j\vert\alpha}
            \leq \Vert f\Vert^2_{\mathcal{H}}\cdot T\alpha
            +\sqrt{8(U^2+Y^2)}\Vert f\Vert_{\mathcal{H}}\cdot T\sqrt{\alpha}.
        $$
        Combining the upper bounds on $\mathcal{T}_{1,1}$, $\mathcal{T}_{1,2}$
        and $\mathcal{T}_2$ concludes the proof.
    \end{proof}

\end{document}